\documentclass{article}

\usepackage{microtype}
\usepackage{graphicx}
\usepackage{subfigure}
\usepackage{booktabs} 

\usepackage{hyperref}

\usepackage{wrapfig}

\usepackage[accepted]{icml2023}

\usepackage{amsmath}
\usepackage{amssymb}
\usepackage{mathtools}
\usepackage{amsthm}
\usepackage{bm, dsfont}

\usepackage{xspace}

\usepackage{tikz}
\usetikzlibrary{arrows,shapes,backgrounds,through,shadows}
\usetikzlibrary{decorations.pathmorphing,calc}

\theoremstyle{plain} 
\newcommand{\thistheoremname}{}
\newtheorem*{genericthm*}{\thistheoremname}
\newenvironment{namedthm*}[1]
  {\renewcommand{\thistheoremname}{#1}
  \begin{genericthm*}}
  {\end{genericthm*}}

\newcommand{\numPar}{{\mathcal P}_Y}
\newcommand{\acro}[1]{\textsc{#1}\xspace}

\DeclareBoldMathCommand{\E}{E}
\DeclareBoldMathCommand{\x}{x}
\DeclareBoldMathCommand{\X}{X}
\DeclareBoldMathCommand{\Pa}{Pa}
\DeclareBoldMathCommand{\v}{v}
\DeclareBoldMathCommand{\V}{V}
\newcommand{\Reals}{\mathbb R}

\newcommand{\df}{\coloneqq}
\newcommand{\prob}{\mathbb{P}}
\newcommand{\logf}[2]{\log\left(\frac{#1}{#2}\right)}
\newcommand\independent{\protect\mathpalette{\protect\independenT}{\perp}} 

\def\independenT#1#2{\mathrel{\rlap{$#1#2$}\mkern2mu{#1#2}}}

\newcommand{\indicator}[1]{\mathds{1}\left\{#1\right\}}

\newcommand{\mnorm}[2]{\left\Vert #1 \right\Vert_{{#2}^{\dagger}}}

\newcommand{\calA}{\mathcal A}

\newcommand{\calN}{\mathcal N}
\newcommand{\calE}{\mathcal E}

\newcommand{\calS}{\mathcal S}
\newcommand{\calP}{\mathcal P}

\newcommand{\ex}{\mathbb{E}}
\newcommand{\graph}{\mathcal G}
\newcommand{\pa}{\text{pa}}
\newcommand{\doi}{\text{do}}

\newcommand{\CBUG}{\acro{CBUG}}
\newcommand{\aCBUG}{a\acro{CBUG}}
\newcommand{\MODL}{\acro{MODL}}

\DeclareMathOperator*{\diag}{diag}
\DeclareMathOperator*{\supp}{supp}
\DeclareBoldMathCommand{\ones}{1}
\renewcommand{\eqref}[1]{Eq. (\ref{#1})}

\newcommand{\cond}{\,|\,}
\usepackage[capitalize,noabbrev]{cleveref}

\theoremstyle{plain}
\newtheorem{theorem}{Theorem}[section]
\newtheorem{fact}{Fact}[section]

\newtheorem{lemma}[theorem]{Lemma}

\theoremstyle{definition}

\newtheorem{assumption}[theorem]{Assumption}
\theoremstyle{remark}
\newtheorem{remark}[theorem]{Remark}

\icmltitlerunning{Additive Causal Bandits with Unknown Graph}

\begin{document}

\twocolumn[
\icmltitle{Additive Causal Bandits with Unknown Graph}

\begin{icmlauthorlist}
\icmlauthor{Alan Malek}{dm}
\icmlauthor{Virginia Aglietti}{dm}
\icmlauthor{Silvia Chiappa}{dm}
\end{icmlauthorlist}

\icmlaffiliation{dm}{DeepMind, London, UK}

\icmlcorrespondingauthor{Alan Malek}{alanmalek@deepmind.com}

\icmlkeywords{Causality, Bandits, Causal Bandits, Causal Inference, Causal Discovery, Linear Bandits, Combinatorial Bandits}

\vskip 0.3in
]

\printAffiliationsAndNotice{}

\begin{abstract}
We explore algorithms to select actions in the causal bandit setting where the learner can choose to intervene on a set of random variables related by a causal graph, and the learner sequentially chooses interventions and observes a sample from the interventional distribution. The learner's goal is to quickly find the intervention, among all interventions on observable variables, that maximizes the expectation of an outcome variable. We depart from previous literature by assuming no knowledge of the causal graph except that latent confounders between the outcome and its ancestors are not present. We first show that the unknown graph problem can be exponentially hard in the parents of the outcome. To remedy this, we adopt an additional additive assumption on the outcome which allows us to solve the problem by casting it as an additive combinatorial linear bandit problem with full-bandit feedback. We propose a novel action-elimination algorithm for this setting, show how to apply this algorithm to the causal bandit problem, provide sample complexity bounds, and empirically validate our findings on a suite of randomly generated causal models, effectively showing that one does not need to explicitly learn the parents of the outcome to identify the best intervention.
\end{abstract}

\section{Introduction}
What setting of our factory production system should we choose to maximize efficiency? Which nutrients would induce maximal crop yield increase? What combination of drugs and dosages would optimize patients outcomes? 
All these questions ask which variables and values would optimize the \emph{causal effect} on an outcome $Y$. In a system of variables $\X_{[K]}=\{X_1,\ldots,X_K\}$ and $Y$, these questions can be phrased as asking which set $\X\subseteq \X_{[K]}$ and values $\x\in\supp(\X)$ would produce an optimal outcome under the \emph{intervention} $\doi(\X=\x)$, whose effect is to alter the \emph{observational distribution} $p(\X_{[K]}, Y)$ describing the existing relationships between $\X_{[K]} \cup Y$ by setting $\X$ to the fixed value $\x$. Indicating with $p(Y\cond\doi(\X=\x))$ the distribution of $Y$ under such an intervention, answering these questions is equivalent to solving the optimization problem $\max_{\X\subseteq \X_{[K]}, \x\in\supp(\X))}\ex[Y\cond\doi(\X=\x)]$.

The causal bandit problem, proposed in \citet{lattimore2016causal}, is an extension of the multi-armed bandit problem to the setting where many variables can be intervened on and a causal graph $\graph$ is used to describe the causal relationships among $\X_{[K]}\cup Y$.  The learner solves this optimization problem by repeatedly choosing an intervention $(\X,\x)$, also called an \emph{action}, and observing a sample from $p(Y\cond\doi(\X=\x))$. A na\"ive approach to the problem would be to treat each of the combinatorially many actions as independent and run a typical bandit algorithm. Instead, the causal graph enables us to reason about acting on sub-parts of the system and exploit the causal structure to reduce the action set $\calA$. For example, consider a system of variables $X_1$, $X_2$, and $Y$ with causal graph $X_1\rightarrow X_2\rightarrow Y$. Fact~\ref{fact1} below tells us that intervening on $\pa(Y)$, the \emph{parents} of $Y$ (or \emph{direct causes}, i.e.\ the variables with an edge into $Y$), can always produce an expectation as high as the best intervention on any other set. As $\pa(Y) = \{X_2\}$, this means that the optimization problem can simplified to $\max_{\x_2\in\supp(X_2)}\ex[Y\cond\doi(\X_2=\x_2)]$.

Causal bandits have been explored under various assumptions on $\graph$, $\calA$, and the interventional distribution. \citet{lattimore2016causal} assumed full knowledge of $\graph$ and of the distribution of $\pa(Y)$ under any intervention in $\mathcal A$ and proposed an algorithm that selects all actions at the start to minimize a lower bound on sample complexity. \citet{lu2020regret} proposed a UCB algorithm for the cumulative regret setting which exploits the observation that one does not need a confidence bound on individual actions but on how the actions affect $\pa(Y)$. \citet{lu2021causal} was the first to considered an unspecified graph and instead considered the special case of a tree $\graph$ with a singleton $\pa(Y)$, which allowed $\pa(Y)$ to be identified via binary search. \citet{de2022causal} did not place constraints on the structure of $\graph$ and instead used causal discovery algorithms to learn a \emph{separating set} (i.e.\ a set that d-separates variables we may intervene on from $Y$) which thereby allowed them to decompose the problem into learning the effect of the actions on the separating set and learning the distribution of $Y$ conditioned on the separating set. This algorithm can be viewed as a generalization of previous causal bandit algorithms which effectively use $\pa(Y)$ as the separating set. \citet{bilodeau2022adaptively} developed an algorithm that performs optimally when a given set is a separating set but fell back to a normal bandit algorithm otherwise. With the exception of \citet{bilodeau2022adaptively}, \citet{de2022causal}, and \citet{maiti2022causal}, these works assumed no \emph{latent confounders} (i.e.\ latent common causes), and many assumed that the learner could make arbitrary interventions on all variables excluding $Y$. \citet{xiong2023combinatorial} were the first to consider the PAC (i.e.\ the fixed-confidence) setting and, similar to this work, provides instance-dependent PAC bounds. They study the known-graph case (potentially with unobserved variables), whereas we focus on the unknown-graph case with the additional assumptions that (i) there are no unobserved confounders between $Y$ and its \emph{ancestors}, and (ii) we may intervene on all variables in $\X_{[K]}$ (which is a common assumption in the causal bandit literature, see \citep{lee2018structural, lu2020regret}).

We refer to this problem as the \emph{causal bandit with an unknown graph} (\CBUG) problem. While the first assumption is self-explanatory and the third assumption common in the literature, our second assumption is a relaxation of the typical no-latent-confounders assumption and natural in several situations. For example, if one has a list that contains the parents of $Y$ (e.g.\ built from domain experts or a noisy causal discovery algorithm), our assumption is still satisfied for arbitrary joint distributions on these variable (even with latent confounding between them and between other variables). Finally, this confounding assumption implies that the optimal intervention set is $\pa(Y)$, as demonstrated in \citet[Proposition~2]{lee2018structural}.
\begin{fact}\label{fact1}
If there are no latent confounders between $Y$ and any of its ancestors, then
\begin{align*}
    \lefteqn{
    \max_{\X\subseteq \X_{[K]}, \x\in \supp(\X)}\ex[Y\cond\doi(\X=\x)]}\\
    &\quad\leq
    \max_{\x'\in \supp(\pa(Y))}\ex[Y\cond\doi(\pa(Y)=\x')].
\end{align*}
\end{fact}
This fact suggests that we can solve the \CBUG problem by first learning $\pa(Y)$ then searching for the optimal values in $\supp(\pa(Y))$; we refer to this approach as \emph{parents-first}. Since a \emph{global intervention} $\doi(\X_{[K]}=\x)$ cuts all incoming edges into $\X_{[K]}$ in $\graph$ such that only the edges from $\pa(Y)$ to $Y$ remain, it suffices to use global interventions and avoid learning the parents.
\begin{fact}\label{fact2}
Under the conditions of Fact \ref{fact1}, 
\[
\ex[Y\cond\doi(\X_{[K]}=\x)]=\ex[Y\cond\doi(\pa(Y)=\x')],
\]
where $\x' = \x \cap \supp(\pa(Y))$ are the values of $\x$ limited to $\pa(Y)$.
\end{fact}
Using the two above facts, the \CBUG problem can be recast as a regression problem over $\supp(\X_{[K]})$. In other words, when there are no latent confounders between $Y$ and any of its ancestors and when we may intervene on all variables, the optimal intervention may be found by a global intervention which circumvents the need to know the causal graph. To our surprise, this observation does not seem to have been used to design algorithms before. 

Unfortunately, even with this simplification, the \CBUG problem can be intractable. In Section~\ref{sec:prelim}, we show that, for any algorithm, there exists a problem instance with sample complexity exponential in the size of $\pa(Y)$; therefore, we need additional structural assumptions to make the problem tractable. Turning towards the causal inference literature, we see that the most common structural assumption is that the outcome is a noisy additive function of its parents, the discrete analog of the linear assumption. It has been developed into its own theory \cite{hastie2017generalized} and is used throughout social and biomedical sciences: see e.g.\ \citet[Chapter~13]{imbens2015causal} and \citet{buhlmann2014cam}), and more recently \citet{maeda2021causal} for examples. This assumption leads us to define the \emph{additive} \CBUG (\aCBUG) problem, where we additionally assume that $Y$ is an additive function of $\pa(Y)$ plus a random term. We emphasize that additive outcome settings are of significant interest to the causal inference community.

As developed in Section~\ref{sec:linear.bandits}, the key implication of the additive assumption is that the \aCBUG problem can be recast as a linear bandit problem (\citet[Chapter~19]{lattimore2020bandit} provides a good introduction) where the action set is combinatorial and we only have full-bandit feedback, meaning that we never observe the effect of individual variables on the outcome and only observe a single sample of the outcome from $p(Y\cond\doi(\X=\x))$.

By considering the specific problem of \aCBUG, we have naturally arrived at the \emph{additive combinatorial linear bandit problem with full-bandit feedback} problem. To the best of our knowledge, we are the first to consider this problem, which extends previous causal bandit settings in two ways: (1) the action set is combinatorial, whereas most prior pure-exploration linear bandits cannot scale to combinatorial actions, and (2) we only have full-bandit feedback, meaning that we can never observe the individual additive components of $Y$ and must infer them from only their sum.

Existing pure-exploration linear bandit algorithms either have complexities that scale with the number of actions or cannot exploit the structure of the problem; hence, in Section~\ref{sec:additive}, we propose a novel action-elimination algorithm that alternates between selecting actions that approximately solve an optimal design problem with decreasing tolerances and using the resulting observations to eliminate suboptimal actions. Noting that storing a combinatorial action set and solving optimal design problems are generally intractable, we solve both computational challenges by restricting $\calA$ to \emph{marginal} action sets that decomposes over variables, which also allows for an easy approximation of the optimal design problem. We name our algorithm \emph{marginal optimal design linear bandit} (\MODL). We analyze the algorithm in the PAC setting and provide one of the first instance-dependent analysis of the sample complexity (with \citet{xiong2023combinatorial} being the only other, to be best of our knowledge). Finally, in Section~\ref{sec:experiments}, we show that \MODL performs well for \aCBUG problems and, in particular, significantly outperforms the parents-first approach while being only slightly behind an oracle version of the algorithm that knows $\pa(Y)$. 

\section{Additive Causal Bandits with Unknown Graphs (\aCBUG)}\label{sec:prelim}
This section gives a formal definition of the \CBUG problem, presents a lower bound showing that any algorithm for solving this problem must have exponential dependence on the number of parents, and introduces the additive outcome assumption.

\textbf{Notation.} 
We refer to sets of variables or values using bold face, e.g. $\X$ and $\x$, respectively. We use a subscript $k$ to indicate variable number, superscripts $i$ or $j$ to indicate a discrete value, and superscripts $n$ or $t$ to indicate sample or round number. For example, $\x_k^t$ is the value of $\X_k$ for the $t$th sample, and $\theta_k^i$ is a parameter corresponding to $X_k$'s $i$th value. Finally, $[n] \df \{1,\ldots, n\}$ and $\x^{[n]}$ indicates a sequence of sets of values.

\subsection{\CBUG Problem Formulation}

We assume a system formed by random variables $\V=\X_{[K]}\cup Y$, where $\X_{[K]}=\{X_1,\ldots,X_K\}$, $\supp(X_k)= \{1,\ldots, M_k\}$ (i.e.\ each $X_k$ has finite integer support), and $Y$ is an outcome of interest that can be real-valued or discrete.
The variables are causally related by an acyclic causal graph $\graph$ with associated \emph{observational distribution} $p(\V)$. 
The learner acts by selecting a set $\X\subseteq \X_{[K]}$ and values $\x\in\supp(\X)$ and performing the  \emph{intervention} $\doi(\X=\x)$, which  corresponds to replacing $p(\V)$ with the \emph{interventional distribution} $p(\V\backslash \X\cond\text{do}(\X=\x))$ resulting from removing all incoming edges into $\X$ from $\graph$ and fixing the value of $\X$ to $\x$. When there are no \emph{unobserved confounders} (i.e.\ a latent common cause between variables in $\V$, usually represented by a bidirected edge in $\graph$),  $p(\V) = p(Y\cond\pa(Y)) \prod_{k=1}^K p(X_k\cond\pa(X_k))$, and the interventional distribution can be expressed as
$p(\V\backslash \X\cond\text{do}(\X=\x))=p(Y\cond\pa(Y))\prod_{k=1 \text{s.t.} X_k\not\in \X}^K p(X_k\cond\pa(X_k))\delta_{\X=\x}$ with $\delta_{\X=\x}$ a delta function centered at $\x$.

The learner's goal in causal bandits is to find the set $\X$ and values $\x$ which result in the greatest expectation of $Y$ under the interventional distribution, denoted $\ex[Y\cond\doi(\X=\x)]$. The learner accomplishes this task by interacting with the environment sequentially, choosing, at every round $t$, 
an intervention $(\X^t,\x^t)$ (also called an \emph{action}) and obtaining a sample from $p(Y\cond\doi(\X^t=\x^t))$. 
Without loss of generality, we assume $\pa(Y) = \{X_1,\ldots, X_{\numPar}\}$ where ${\mathcal P}_Y$ is the number of parents of $Y$ (of course the learner does not know this ordering).

We consider the causal bandit problem in the setting where $\graph$ is unknown and the learner must find, for $\epsilon>0$ and $\delta\in(0,1)$, an $(\epsilon, \delta)$-PAC solution $(\hat\X,\hat\x)$  satisfying
$
    \prob\Big(\max_{\X,\x} \ex[Y\cond\doi(\X=\x)] - \ex[Y\cond\doi(\hat\X=\hat\x)] \leq \epsilon\Big) \notag\\
    \quad\geq 1-\delta.
$
in as few rounds (or samples, we use the two interchangeably) as possible. We refer to this problem as the \emph{causal bandit with an unknown graph} (\CBUG) problem. 

\subsection{Global Intervention Approach} Recall that the \CBUG problem assumes no latent confounders between $Y$ and any of its \emph{ancestors} (i.e.\ variables with a \emph{directed (or causal) path} into $Y$) and that we can simultaneously intervene on all observable variables except for $Y$, i.e.\ we can make \emph{global interventions}. A discussed previously, global interventions are a common model in the literature. They also model settings where statistical units are expensive relative to the cost of intervening on additional variables for a single unit, which implies that the sample complexity (the number of units used) is the key quantity to minimize, not the number of variables intervened on.

As stated in Fact \ref{fact1}, these two assumptions imply that $\pa(Y)$ is the optimal intervention set. Thus, a natural solution to solving the \CBUG problem would be to first find $\pa(Y)$ and then search for the optimal value in $\supp(\pa(Y))$. Instead, we make the key observation that a \emph{global intervention} $\doi(\X_{[K]}=\x)$  cuts all incoming edges into $\X_{[K]}$ in $\graph$, leaving only the edges from $\pa(Y)$ to $Y$, which implies that performing a global intervention is equivalent to intervening on $\pa(Y)$; precisely, $\ex[Y\cond\doi(\X_{[K]}=\x)]=\ex[Y\cond\doi(\pa(Y)=\x')]$, where $\x'$ are the values of $\x$ for $\pa(Y)$ (recall Fact~\ref{fact2} above). This claim can be proved by invoking rule 3 of do-calculus \cite{pearl2000causality} which, in this case, states that $\ex[Y\cond\doi(\X_{[K]})]=\ex[Y\cond\doi(\pa(Y))]$ since $Y\independent_{{\cal G}_{\bar \X_{[K]}}} \X_{[K]}\backslash \pa(Y)\cond\pa(Y)$, where ${\cal G}_{\bar \X_{[K]}}$ is the graph $\graph$ with all incoming edges into $\X_{[K]}$ removed. We can therefore restrict the problem to finding $\hat x\in \supp(\X_{[K]})$ that satisfies 
$\prob\Big(\max_{\x\in \supp(\X_{[K]})} \ex[Y\cond\doi(\X_{[K]}=\x)] - \ex[Y\cond\doi(\X_{[K]}=\hat\x)] \leq \epsilon\Big) \geq 1-\delta$, meaning that learning the parents, in some sense, is optional. 

\subsection{The \CBUG Problem is Exponentially Hard}
While using global interventions saves us from having to consider all intervention sets in the powerset of $\X_{[K]}$, the \CBUG problem can be exponentially hard in ${\mathcal P}_Y$.
For a fixed $\delta$ and $\epsilon$, let $\x^*\in\supp(\pa(Y))$ be fixed and unknown, and let  $\ex[Y\cond\doi(\pa(Y)=\x)] = 0 + \epsilon \indicator{\x=\x^*}$: the expectation is flat except at a single value $\x^*$ where it is equal to $\epsilon$. We can strengthen the example by choosing $p$ such that, for any intervention $\doi(\X'=\x')$ with $\x'\supseteq \x^*$ (indicating that $\x'$ agrees with $\x^*$ in $\supp(\pa(Y))$) we have $p(\X_k=\x_k^*\cond\doi(\X'=\x'))=0$ for all $\X_k\notin \X'$. Only interventions with $\x'\supseteq \x^*$ can provide information about $\x^*$, so this problem is difficult as the learner has to try actions blindly until one containing $\x^*$ is found. We assume that $Y\cond\doi(\X=\x)$ is 1-sub-Gaussian for any intervention.

Obtaining an upper bound on this problem is easy. Consider the algorithm that picks a ordering all values $\x^1, \x^2\ldots$ in $\supp(\X_{[K]})$ uniformly at random. Beginning at $t=1$, it collects $O\left(\frac{\sum_k \log(M_k)}{\epsilon^2}\logf{1}{\delta}\right)$ samples from $p(Y\cond\doi(\X_{[K]}=\x^t)$ and tests $\ex[Y\cond\doi(\X=\x)] \geq \epsilon$ against the null hypothesis $\ex[Y\cond\doi(\X=\x)] = 0$. If the null hypothesis is rejected, then $\x^t$ is optimal and the algorithm terminates; otherwise, the algorithm moves on to $t+1$. The sample complexity is $O\left(\left|\supp(\pa(Y))\right|\frac{\sum_k \log(M_k)}{\epsilon^2}\logf{1}{\delta}\right)$: since $\prob(\x^*\in \x^t) = 1/|\supp(\pa(Y))|$, we have to test, on average, $O\left(|\supp(\pa(Y))|\right)$ actions before stumbling upon one containing $\x^*$. This na\"ive algorithm matches the following lower bound with proof in Appendix~\ref{appendix:proofs}.
\begin{theorem}\label{thm:lower.bound}
There is an instance of the CBUG problem such that any $(\epsilon,\delta)$-PAC algorithm must take
$\Omega\left(
\frac{\left|\supp(\pa(Y))\right|}{\epsilon^2}\logf{1}{\delta}\right)$
samples in expectation.
\end{theorem}

\subsection{Additive Outcome Assumption}
As the example from the previous section illustrates, a problem where information about the optimal intervention is hyper-localized (i.e.\ where one only learns about the optimal intervention by trying it) is information-theoretically difficult. Therefore, we need some assumptions to make the problem tractable and, as discussed in the introduction, turn to the additive assumption from the  causal inference literature \cite{buhlmann2014cam}.
\begin{assumption}[Additive Outcome]\label{additive}
There exist functions $f_1,\ldots, f_{\numPar}$ and a $\sigma^2$-sub-Gaussian random variable $\eta$ such that $Y = \sum_{k=1}^{\numPar} f_k(X_k) + \eta$.
\end{assumption}
This assumption implies that the causal effect of $\pa(Y)$ on $Y$ decomposes into the sum of individual effects from each parent, which results in an \emph{additive} \CBUG (\aCBUG) problem. We focus on the homoscedastic case where $\eta$ is an i.i.d. $\sigma^2$-sub-Gaussian random variable, by far the most common assumption in the literature.

\section{Pure-Exploration Linear Bandits}\label{sec:linear.bandits}
This section defines the additive combinatorial linear bandit with full-bandit feedback problem and shows how \aCBUG is a special case. With the additive outcome assumption, the CBUG problem can be cast as a pure-exploration linear bandit problem with a combinatorial action set $\calA$. The linear bandit problem is a sequential decision problem where, at round $t$, the learner chooses an action $\x^t\in\calA$ and observes $y^t = \langle \x^t, \theta^*\rangle+\epsilon^t$, where $\epsilon^t$ is a zero-mean $\sigma^2$-sub-Gaussian random variable and $\theta^*\in\Reals^d$ is a fixed but unknown parameter. The goal of the learner is to find an $(\epsilon,\delta)$-PAC action in a few rounds/samples as possible.

We cast \aCBUG as a linear bandit problem using one-hot-encoding. For $k\in [K]$, let $e_k(i)$ be the $i$th unit vector in $\Reals^{M_k}$, and for $\x\in\supp(\X_{[K]})$, define $e(\x) = (e_1(\x_1),\ldots, e_K(\x_K))$ to be the concatenation of the one-hot vectors, which produces a mapping from $\supp(\X_{[K]})$ to $\Reals^d$ with $d\df\sum_k M_k$. Defining the vector $\theta = (f_1(1),
f_1(2),\ldots, f_K(M_K))$, we obtain $\ex[Y\cond\doi(\X_{[K]} = \x)] = \langle \theta^*, e(\x)\rangle$. Therefore, the goal is to find $\arg\max_{\x\in\supp(\X_{[K]})} \langle \theta^*, e(\x)\rangle$. Note that the terms of $\theta^*$ corresponding to $X_k\notin\pa(Y)$ are zero, so $\theta^*$ is sparse. It is important to note that we only have full-bandit feedback because we observe $ \langle \theta^*, e(\x)\rangle$ and not the individual $f_k$ components. Even though the action set $\mathcal A$ has a combinatorial structure with size $\prod_{k=1}^K M_k$, as it is the Cartesian product of choosing one value for each variable, the additive assumption allows us to consider the dimension-$d$ linear problem instead.

Casting \aCBUG as a linear bandit problem enables us to borrow from the extensive literature on pure-exploration linear bandits.  One successful approach has been to treat the action selection as a optimal experimental design problem: that is, selecting actions to reveal as much information about a hidden parameter vector estimated through regression. While these optimal design problems tend to be intractable except for special cases, we show how to approximate our action set to avoid these difficulties. 

This approach was pioneered by \citet{soare2014best}, who had the insight a that the optimal design problem should optimize for learning the gaps between actions. Improvements in sample complexity were made by \citet{xu2018fully} and \citet{tao2018best} by using a different estimator and a different design approximation strategy, respectively, while \citet{fiez2019sequential} considered the more general problem of transductive experimental design. A survey of optimal design in linear bandits can be found in \citet[Chapter~22]{lattimore2020bandit}. Unfortunately, all of these algorithms have complexity that is linear in the number of actions and are therefor not tractable for our combinatorial action space. Another line of work, \cite{chen2014combinatorial,gabillon2011multi}, had the same additive action structure as us but assumed semi-bandit feedback, i.e. where noisy observations of individual $f_k(x_i)$ are possible. Because we only observe $y^t = \sum_{k=1}^{\numPar} f_k(\x_k^t) + \eta$, we are in the full-bandit setting and cannot use these algorithms either.

The only work we are aware of in the pure-exploration combinatorial linear bandit with full-bandit feedback setting is by \citet{du2021combinatorial}, who claimed the first efficient algorithm for this setting. Their approach uses a pre-sampling step to select a subset of the actions of size $O(poly(d))$ and then runs the algorithm of \cite{constantinou2017extended}. The resulting algorithm is fairly complex (requiring multiple sub-procedures including an entropy mirror-descent stage) and requires finding a size-$d$ subset of actions with rank $d$. In our case, the rank of any subset of actions is at most $d-1$ so we cannot use this algorithm. Further, their algorithm is general purpose and cannot fully exploit the structure of our action space. Hence, we created a new algorithm, introduced in the following section.

\section{Marginal Optimal Design Linear Bandit}\label{sec:additive}
Given data $\{(\x^t,y^t)\}_{t=1}^n$, we need to learn about the unknown parameter vector $\theta^*$ in a way that lets us quantify the uncertainty.  
With $V_n = \sum_{t\leq n} \x^t (\x^t)^\top$  denoting the data covariance and $V_n^\dagger$ its pseudoinverse, we use the ordinary least squares (OLS) estimator, $\hat\theta = V_n^\dagger \sum_{t\leq n} \x^t y^t$, which has the following Azuma-style confidence interval for $\hat\theta$ (see, e.g.\ \citet{soare2014best}):
\begin{lemma}\label{lem:azuma}
Let $\hat\theta$ be the OLS estimator calculated from data $\x^{[n]}$ with covariance matrix $V_n$.
For any $z\in\Reals^d$, $\delta\in(0,1)$, and for $\sigma^2$-sub-Gaussian $\eta$,
$
    \prob\left(
        \langle \hat\theta - \theta^*, z\rangle \geq \sqrt{2\sigma^2\mnorm{z}{V_n}^2\log(1/\delta)}  
    \right)
    \leq \delta.
$
\end{lemma}
This proof, as well as all other omitted proofs, are given in Appendix~\ref{appendix:proofs}.
The lemma requires $\x^{[n]}$ to not be a function of the data. The main challenge in using this inequality is that we need to solve for a sequence of covariates to minimize the bound.

Following \citet[Chapter~22]{lattimore2020bandit}, we propose an action-elimination algorithm that proceeds in phases. Each phase begins with a set $\calS$ of plausibly best actions and a desired tolerance $\gamma$. The algorithm chooses actions to optimize the upper bound in Lemma~\ref{lem:azuma}, then uses this guarantee to prune $\calS$. The tolerance is then decreased before the next phase. Phases repeat until an optimal action is identified or all sub-$\epsilon$ actions have been removed.  There are two main computational difficulties. First, the number of actions, $|\mathcal A|=\prod_k M_k$, is very large, which makes the pruning step potentially intractable. We need algorithms that scale with the exponentially smaller ambient dimension $d=\sum_k M_k$. Second, the action selection (known as an optimal design problem) is a combinatorial optimization problem and generally difficult. We solve both problems at once by limiting $\mathcal S$ to sets that decompose over variables, defined below.

\textbf{Marginal Action Sets.} 
We say that a set of actions $\calS\subseteq {\mathcal A}=\supp(\X_{[K]})$ is \emph{marginal} if there exist $\calS_k\subseteq\supp(X_k)$, $k=1,\ldots, K$ such that
$
    \calS = 
    \left\{
     (j_1, j_2,\ldots, j_K):
     j_1\in \calS_1,
     j_2\in \calS_2,
     \ldots,
     j_K\in \calS_K
     \right\}
$. In other words, $S$ consists of the Cartesian product of $\calS_1,\ldots, \calS_K$. Marginal action sets are intuitive: if we have eliminated, say, $X_k = j$ as a good action, then we should never consider any action where $X_k = j$. Marginal action sets solve the combinatorial action set problem, since such sets can be represented by $\sum_k M_k$ binary values.

\textbf{Optimal Design.} 
In linear bandits, our goal is to choose actions $\x^t$ to reveal as much about the optimal action as possible. While the most obvious goal would be to choose actions to minimize a bound on $\langle\hat\theta-\theta^*, e(\x)\rangle$ simultaneously for all actions $\x\in\calS$, estimating the \emph{gaps} between actions $\x$ and $\x'$, defined as $\Delta(\x, \x') \df \langle\theta^*, e(\x)-e(\x')\rangle$, is more efficient \citep{fiez2019sequential}. Defining $V_n \df \sum_{t=1}^n e(\x^t)e(\x^t)^\top$, the optimal design problem is 
\begin{align}\label{eqn:optimal.design}
   \arg\min_{\x^{[n]}}   
   \max_{\x, \x'\in\calS} \mnorm{e(\x)-e(\x')}{V_n}.
\end{align}
Generally, the optimal design problem (\ref{eqn:optimal.design}) is intractable \cite{xu2018fully}, and the state-of-the-art algorithms are linear in $\calS$ \cite{allen2021near}. Fortunately, marginal action sets afford a computationally simple solution with an easy to calculate upper bound.
\begin{lemma}\label{lem:optimal.design.solution}
Assume that $\calS$ is marginal, and let $\tilde\x^{[n]}$ be any sequence of actions that are uniform in every marginal, i.e.\ for every $k$, $\sum_{t=1}^n\indicator{\x_k^t = i} - \indicator{\x_k^t = j}\leq 1$ for all  $i,j\in \calS_k$. With $\tilde V_n$ as the covariance matrix of $\tilde\x^{[n]}$, we have
\[
    \max_{\x,\x'\in\calS}\mnorm{e(\x)-e(\x')}{\tilde V_n}\leq \sum_i \frac{2|\calS_i|}{n-|\calS_i|}\leq \sum_i \frac{2|\calS_i|}{n}.
\]
\end{lemma}
Roughly, the proof proceeds by noting that $\tilde V_n$ can be written as a diagonal matrix of counts plus the cross terms, both of which are positive semi-definite. We can upper bound the total expression the norm defined only with the diagonal terms, which permits a particularly simple form of $\mnorm{e(\x)-e(\x')}{\tilde V_n}$ that we can explicitly calculate for uniform sequences of marginals. 

\begin{remark}
The embedding that we use for the linear bandits is not full rank. For example, for any vector $v\in\Reals^K$ with $\ones^\top v = 0$, the null space of $V_n$ includes $(v_1 \ones(M_1), \ldots, v_K\ones(M_K))$ (where $\ones(n)$ is the ones vector of length $n$). However, what is important is that the projection of the nullspace onto the coordinates in $\calS_k$ is always in the all-ones direction, which allows us to calculate unbiased estimates of the gaps, even though we may not be able to identify $\theta$. This insight provides another reason why \eqref{eqn:optimal.design} is the correct optimization objective.
\end{remark}

\subsection{Deriving the Elimination Algorithm}\label{sec:rejection.rule}
At each phase of the algorithm, we have an action set $\calS$ and an error tolerance $\gamma$, and we wish to find a set of actions $R$ to remove from $\calS$ that can be guaranteed to be suboptimal. The crux is that we can only calculate the empirical gaps $\hat\Delta(\x, \x') \df \langle\hat\theta, e(\x)-e(\x')\rangle$, thus we need to bound the error $\hat\Delta(\x, \x') - \Delta(\x, \x')$. 

Suppose that we choose $\x^{[n]}$ according to Lemma~\ref{lem:optimal.design.solution}; Lemma~\ref{lem:azuma} then guarantees that 
$
\langle \hat\theta - \theta^*, e(\x)-e(\x')\rangle\leq \sqrt{4\sigma^2 \frac{\sum_k |\calS_k|}{n}\logf{1}{\delta}}
$
holds with high probability for all $\x, \x'\in\calS$. This means that it suffices to take $n = \left\lceil \frac{4\sigma^2 |\sum_k\calS_{k}|}{\gamma^2}\logf{L}{\delta}\right\rceil$
 if we want to ensure that 
$\langle\hat\theta - \theta^*, e(\x)-e(\x')\rangle\leq\gamma$.

The usual choice in the bandit literature is $R = \{ \x\in\calS: \exists \x'\in\calS \text{ s.t. } \hat\Delta(\x,\x')\geq\gamma\}$. Letting $\x^*$ be the optimal action, we see that
\begin{align*}
    \langle \theta^*, e(\x)-e(\x^*)\rangle
    &\leq
    \langle \theta^*, e(\x)-e(\x')\rangle\\
    &\leq 
    \langle \hat\theta, e(\x)-e(\x')\rangle - \gamma
    \leq 
    0
\end{align*}
for all $\x\in R$. Using the inequality in the other direction, any $\x\in \calS \setminus R$ must have
$
    \langle \theta^*, e(\x)-e(\x^*)\rangle
    \leq
    \langle \theta^*-\hat\theta, e(\x)-e(\x^*)\rangle
+     \langle \hat\theta, e(\x)-e(\x^*)\rangle
    \leq
    2\gamma
$.

\begin{algorithm}[tb]
   \caption{\MODL}
   \label{algorithm:MODL}
\begin{algorithmic}
   \STATE {\bfseries Input:} $\delta>0$, $\epsilon>0$, $\X_{[K]}$, $\sigma^2$, $B$
\STATE Optional: upper bound $\overline\calP_Y$ on $\calP_Y$
\STATE $\calS_k(1) \leftarrow [M_k]$ for $k=1,\ldots, K$
\STATE $L\leftarrow \left\lfloor \logf{2BK}{\epsilon}\right\rfloor$
\FOR{$\ell=1,\ldots, L$}
    \STATE $\gamma(\ell) \leftarrow\frac{\epsilon}{K} 2^{L-\ell+1}$
    \STATE $n \leftarrow \left\lceil \frac{4\sigma^2 |\sum_k\calS_k(\ell)|}{\gamma(\ell)^2}\logf{L}{\delta}\right\rceil$
    \STATE Choose $\x^1,\ldots, \x^n\in\supp(\X_{[K]})$ using Lemma~\ref{lem:optimal.design.solution}
    \STATE Collect $y^t \sim P(Y\cond\doi(\X_{[K]} = \x^t))$ $\forall t\leq n$
    \STATE Update $V_n$, $\hat\theta(\ell) = V_n^{\dagger} \sum_{t\leq n}\x^t y^t$
    \STATE Calculate empirical gaps $\hat\Delta_k^j$
    \FOR{$k=1, \ldots, K$}
    \STATE $\calS_k(\ell+1) \leftarrow \left\{j\in \calS_k(\ell): \hat\Delta_k^j < \gamma(\ell)\right\}$
    \ENDFOR
    \STATE $\hat \calP_Y(\ell)\leftarrow\sum_k \indicator{|\calS_k(\ell)| = 1}$\\
    \IF{$\hat \calP_Y(\ell+1) = K$ or $\hat \calP_Y(\ell)\geq \overline\calP_Y$}
    \STATE Break
    \ENDIF
\ENDFOR
\STATE Return $\arg\max_{\x\in \calS} \langle\hat\theta(\ell), e(\x)\rangle$ and $\hat\theta(\ell)$
\end{algorithmic}
\end{algorithm}

This rejection procedure is guaranteed, with probability at least $1-\delta$, to eliminate all $2\gamma$-suboptimal actions and never eliminate the optimal action.

The reader may notice that $\calS\setminus R$, is not marginal even with this choice of $R$, even if $\calS$ is marginal. Instead, we want a $R$ such that 1) $\calS\setminus R$ is marginal, and 2) $R$ is as large as possible. Such a marginal-preserving rejection rule is necessary for the tractability of \eqref{eqn:optimal.design}.

Since we require $\calS'\df \calS\setminus R$ to be marginal, we can define $R_k\df \calS_k \setminus\calS'_k$ to be the values removed from $X_k$'s marginal. How must we constrain $R_k$ so that, for every $\x\in R$, there must be some $\x'\in\calS$ with $\hat\Delta(\x, \x')\geq\gamma$?  

We use the fact that gaps decompose by variables: if we define
$\hat\Delta_k^{i,j} \df \hat\theta_k^i - \hat\theta_k^j$ and 
$\hat\Delta_k^{j} \df \max_{i}\hat\theta_k^i - \hat\theta_k^j$,
then the gap between $\x$ and $\x'$ decomposes as
$\hat\Delta(\x, \x') = \sum_{k\in [K]} \hat\Delta_k^{\x_k, \x'_k}
$. Taking $\x'= \x^*$, we have that $\hat\Delta(\x,\x^*) = \sum_k\hat\Delta_k^{\x_k}$. Thus, we may only include $i\in R_k$ if, for all $\x\in\calS$ with $\x_k = i$, we can guarantee that $\hat\Delta(\x,\x^*) = \sum_k \hat\Delta_k^{\x_k}\geq\gamma$. Using $\x^*(i)$ to denote $\x^*$ with the $k$th value set to $i$, it is easy to check that $\hat\Delta(\x^*(i),\x^*) = \hat\Delta_k^i$, which implies that we can only include $i$ in $R_k$ if $\hat\Delta_k^i \geq \gamma$.

\begin{figure*}[t]
\centering
\includegraphics[width=14.3cm,height=6.5cm]{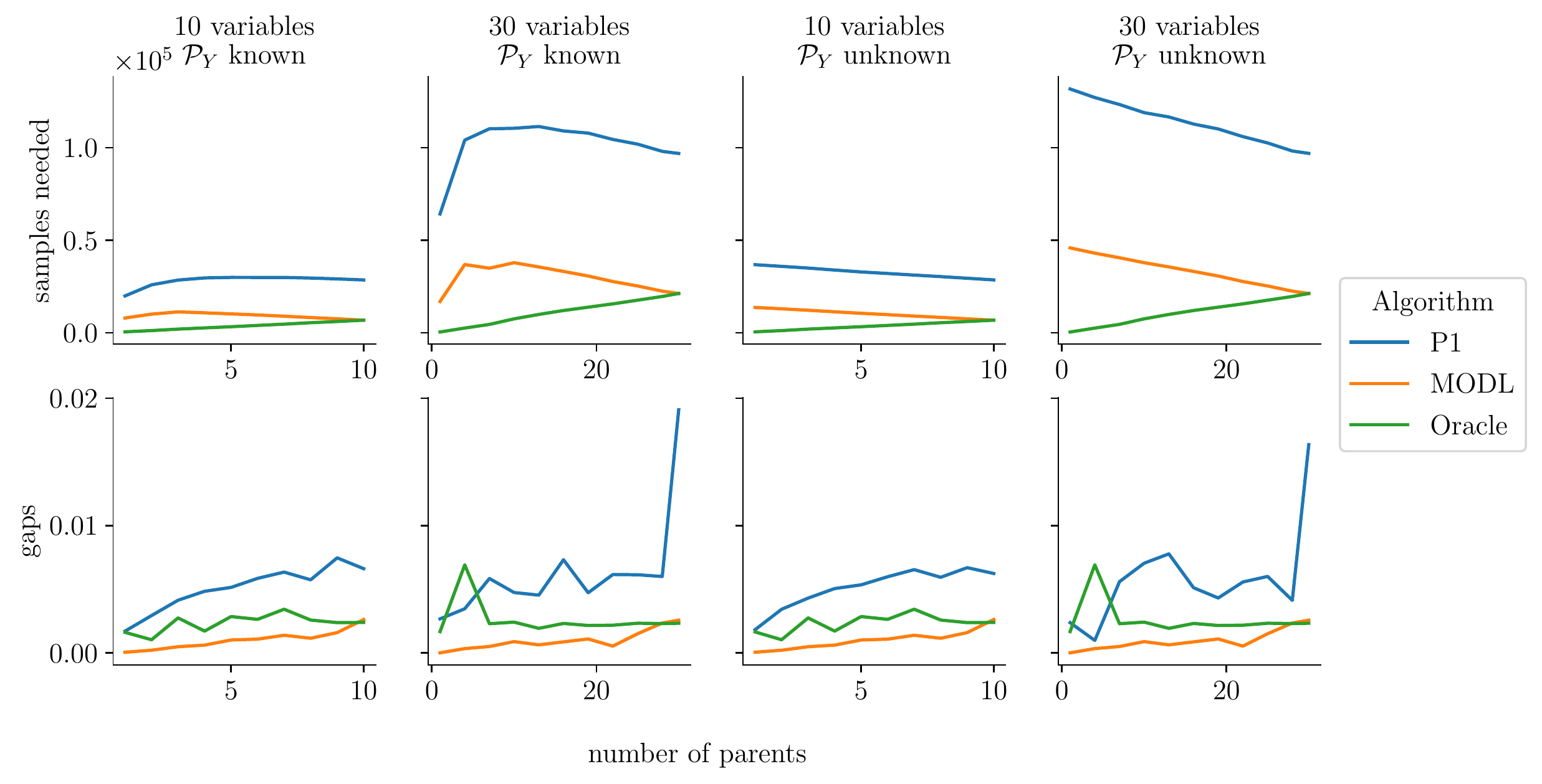}
\caption{Sample complexity and average gaps versus number of parents of $Y$.} 
\label{fig:var_parents}
\end{figure*}
\subsection{The \MODL Algorithm}
With the optimal design and rejection procedures derived, we can present the \emph{marginal optimal design linear bandit} (\MODL) algorithm and its sample complexity bound. \MODL proceeds in phases $\ell=1,\ldots, L$, and in each phase it solves an $\mathcal{XY}$-optimal design problem, using the results of Lemma~\ref{lem:optimal.design.solution} with error $\gamma = \epsilon 2^{L-\ell}$, ensuring that $\gamma=\frac{\epsilon}{2}$ by the time the algorithm terminates. The algorithm uses the rejection rule of Section~\ref{sec:rejection.rule} to maintain a marginal action set $\calS$. We also consider the case when $\calP_Y$ is provided, which allows termination once $\calP_Y$ variables have their optimal value identified. The intuition is that $\theta_k^1, \ldots, \theta_k^{M_k}$ are approximately equal for all for $X_k\not\in\pa(Y)$, so the algorithm is not able to limit $\calS_k$ to a single value. Pseudocode is provided in Algorithm \ref{algorithm:MODL}.

We remark that restricting to marginal action sets does not eliminate any action. Instead, this restriction potentially reduces the number of actions that can be eliminated by requiring that the set of remaining actions be expanded to the smallest marginal action set containing it. In essence, marginal action sets allow us to trade-off some statistical efficiency for computational tractability.

We analyzed the expected sample complexity of the algorithm and present an upper bound in Theorem~\ref{thm:sample.complexity}. We find the typical sum-of-reciprocal-squared-gaps dependence common to best-arm-identification problems, $O\left(\sum_{i,k}(\Delta_k^i\vee\epsilon)^{-2}\right)$, however, instead of a sum over the combinatorial action set, the sum is over the all gaps for individual variables, which is the sample complexity one would expect if each variable could be played independently. In other words, despite only having full-bandit feedback, we obtain the sample complexity as if we had semi-bandit feedback. A substantial part of the complexity comes from the $\sum_{k\notin\pa(Y)} M_k \epsilon^{-2}$ term, which arises because the non-parents are the most difficult: to differentiate between the cases when a variable is a non-parent or when there is a single value that is $\epsilon/K$ better than the rest, all values must be estimated within $\epsilon/K$.

For the known-$\calP_Y$ case, the $(\Delta_k^i\vee(\epsilon/K))^{-2}$ term is replaced by
$(\Delta_k^i\vee\Delta_{\min}\vee(\epsilon/K))^{-2}$ which could be substantially smaller if $\Delta_{\min} \gg \epsilon/K$. We see this reduction because we no longer need to identify the non-parents, but rather can terminate once the minimum gap among the parents is found.
\begin{theorem}\label{thm:sample.complexity}
Algorithm~\ref{algorithm:MODL} is $(\epsilon,\delta)$-PAC. The expected sample complexity has an upper bound of
\begin{align*}
    \frac{16\sigma^2}{3} \logf{\logf{BK}{\epsilon}}{\delta}
    \sum_{k=1}^K
    \sum_{i=1}^{M_k}
    \frac{1}{(\Delta_{\min} \vee \Delta_k^i\vee\frac{\epsilon}{K})^2},
\end{align*}
where $\Delta_{\min} = \min_{k\leq \calP_Y}\min_{i\in [M_k]}\Delta_k^i$ is the minimum gap in the parents in the case when number of parents $\calP_Y$ is provided, and $0$ otherwise.
\end{theorem}

Due mostly to the additive assumption, the sample complexity contains $\sum_k M_k$ terms of order $O\left((\Delta \vee \epsilon)^{-2}\right)$, which is the same order as the sample complexity of running a separate bandit algorithm for each variable despite only observe the sum of rewards. In contrast, a naive approach which, ignoring the structure, simply uses a best-arm-identification algorithm over the combinatorial action set would have $\prod_k M_k$ terms in the complexity bound of order
$\sum_{j_1=1}^{M_1}\cdots
\sum_{j_K=1}^{M_K}
 \frac{1}{(\Delta_{\min} \vee (\sum_k \Delta_{k}^{j_k})\vee\epsilon)^2}$.

\textbf{Recovering the Parents of $Y$.}
With simple assumptions on $f_k$, we may recover a good estimate for $\pa(Y)$ upon termination of the algorithm. Using the parameter estimates returned by the algorithm, we define $\widehat\pa(Y)$ to be all the nodes $X_k$ where $ \forall \ell,  |\hat\theta_k^i(\ell)-\hat\theta_k^j(\ell)|\leq 2\gamma(\ell)
\;\;\forall i,j\in\calS_k(\ell)$.
This formula follows the intuition that non-parents $k$ have all $\theta_k^j$ identically equal and thus $\hat\theta_k^j(\ell)$ should be within the error tolerance $\gamma(\ell)$. We can show that this method works with high probability, provided an identifiability condition holds. Without any identifiability assumptions, no algorithm can be guaranteed to recover the parents.
\begin{theorem}\label{thm:recovering.parents}
Assume that there is some $\epsilon_{\min}>0$ such that, for all $k\leq \calP_Y$, there exist $i, i'\in[M_k]$ with $|f_k(i)-f_k(i')|\geq \epsilon_{\min}$. Let $\hat\x$ and $\hat\theta$ be the output of Algorithm~\ref{algorithm:MODL} run with $\epsilon\leq\epsilon_{\min}$ and $\delta>0$. Then $\widehat\pa(Y)$ as defined above has
$    \prob(\widehat\pa(Y) = \pa(Y))\geq 1-\delta$.
Furthermore, the intervention $\hat\x_{\widehat\pa(Y)}\df \{\X_i = \hat\x_i: \X_i\in\widehat\pa(Y)\}$, which is $\hat\x$ limited to $\pa(Y)$, is $(\epsilon, \delta)$-PAC.
\end{theorem}

\section{Experiments}\label{sec:experiments}
This section presents an empirical evaluation of the \MODL algorithm on a collection of randomly generated causal additive models\footnote{Code has been released at \url{https://github.com/deepmind/additive_cbug}.}. Additional experiments studying the effect of graph structure and the sensitivity to the additive outcome assumption's violation can be found in Appendix~\ref{appendix:experiments}.

\textbf{Baselines.} 
Since we are the first to consider the general setting of unknown $\graph$  without assumptions on its structure, it is difficult to compare \MODL to other algorithms in the causal bandit literature. The closest algorithms are those of \citet{de2022causal} and \citet{bilodeau2022adaptively} with the separating set taken to be all intervenable random variables. In this settings, these algorithms reduces to a multi-armed bandit on the full, product action space. As the number of actions is exponential in the number of variables, we were only able to include this baseline for experiment with few variables. Since these algorithms were designed for the cumulative regret setting, we have implemented a version using Successive Elimination (SE) \cite{even2006action}. 

We also compare \MODL to (i)
a \emph{parents-first (P1) method} which first performs hypothesis testing to find an approximate parents set $\hat\pa(Y)$ and then runs Algorithm~\ref{algorithm:MODL} with $\X_{[K]}=\hat\pa(Y)$ (i.e.\ considering $\hat\pa(Y)$ as the intervention set), and (ii) an \emph{oracle method} which runs Algorithm~\ref{algorithm:MODL} with $\X_{[K]}=\pa(Y)$. Fact~\ref{fact2} guarantees that intervening on $\pa(Y)$ alone is sufficient to solve the problem; therefore, the difference in performance between \MODL and the oracle method quantifies the value of knowing the parents. Comparing \MODL with the P1 method answers whether spending samples to explicitly learn the parents is efficient.

For learning $\pa(Y)$, we were not able to find any suitable algorithm in the literature that exploits the ability to intervene on all $\X_{[K]}$. Thus, we designed our own algorithm for finding an approximate parents set $\hat\pa(Y)$ using global interventions. Let $\x_0\in \supp(\X_{[K]})$ be some fixed intervention; for each $k$ in some random order, we enumerate $j\in[M_k]$ and test a null hypothesis of $\ex[Y\cond\doi(\X_{[K]} = \x_0)] = \ex[Y\cond\doi(X_k = j, \X_{[K]}\setminus\{X_k\} = \x_0)]$; $X_k$ is added to $\hat \pa(Y)$ only if we find a $j$ where the null hypothesis is rejected. We terminate early if $\hat\pa(Y)$ is large enough to meet a bound on $\calP_Y$. Provided that, for all $X_k\notin\pa(Y)$, there exists some $j$ with $|f_k(j)|\geq\epsilon$, this algorithm is guaranteed to find $\pa(Y)$ with high probability. Pseudocode and a complexity bound are provided in Appendix~\ref{appendix:parents.test}. 

\textbf{Experimental Set-up.}
We performed the evaluation on randomly sampled structural causal models (SCMs) generated as followed. The causal graph, excluding $Y$, was a sampled directed acyclic graph from the Erd\H{o}s-R\`{e}nyi 
model with the degree 3 and $K-1$ variables. We randomly choose set of variables of size $\calP_Y$ as the parents of $Y$, then each variable topologically greater than $Y$ is independently set as a child to $Y$ with probability $.5$.

To sample the conditional probability distributions, we chose $M_k$ uniformly between specified upper and lower bounds and generated the conditional probability distribution for each $X_k$ by sampling $p(X_k=j\cond\pa(X_k) = \x)\propto \mathrm{Beta}(2,5)$ independently for all $j$ and $\x$. Finally, we generated $f_k(j) = B W_k^j$, with $B=5$ and $W_k^j$ sampled i.i.d. from $\mathrm{Beta}(2, 5)$ and set $\eta$ to a standard normal variable. If $X_j$ had $Y$ as a parent, we used the same construction but with $Y$ rounded to an integer.

\textbf{Results.}
Using $\epsilon = 1/2$ and $\delta = .1$ for our $(\epsilon,\delta)$-PAC criterion, we considered a variety of settings of $\calP_Y$, $K$, and upper and lower bounds for $M_k$. Each point all graphs corresponds to the average over 20 different SCMs sampled using the process described above and 50 independent runs of the methods on independently generated data.

\begin{figure}[H]
\centering
\includegraphics[width=6.5cm,height=5cm]{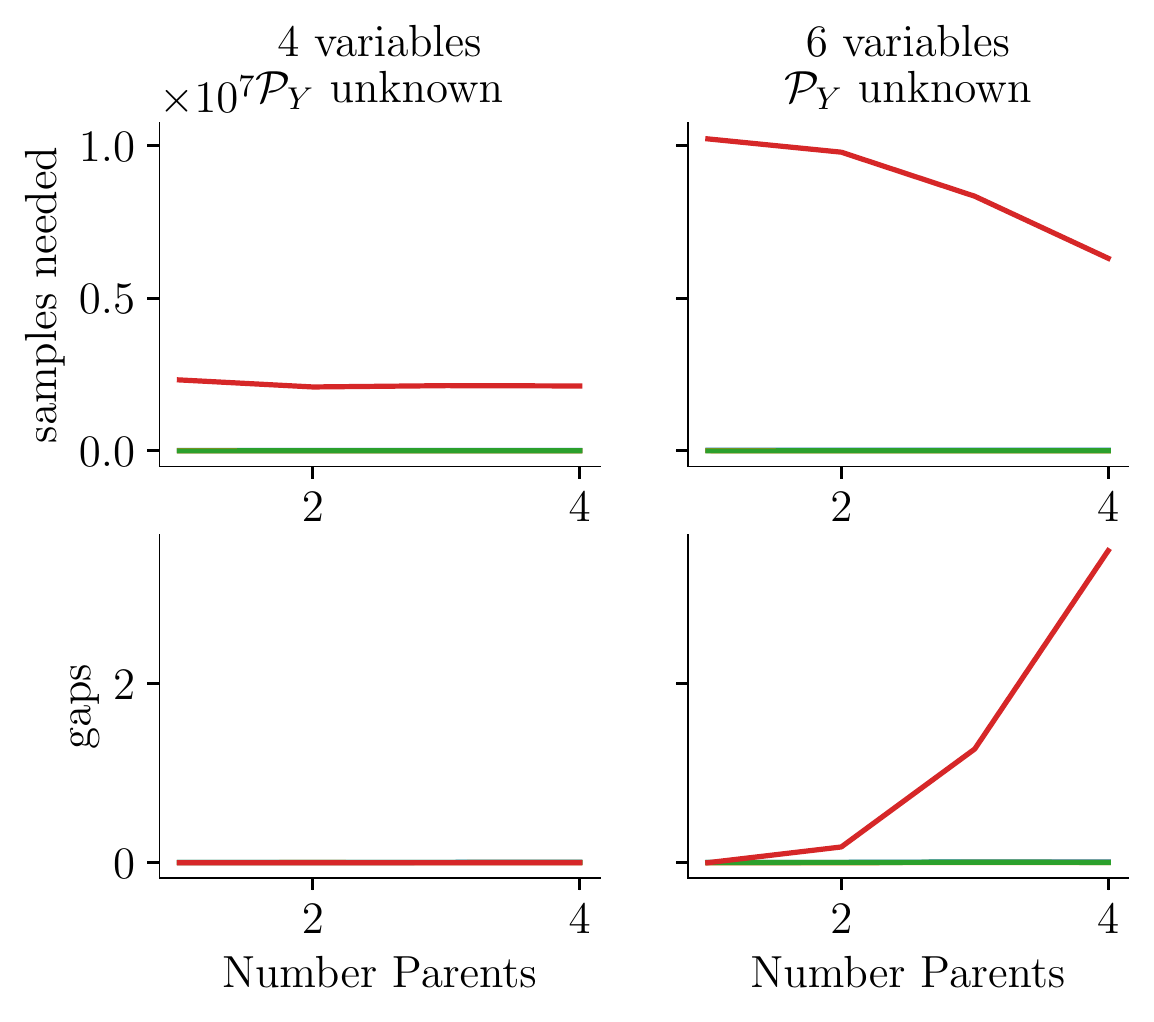}
\caption{Sample complexities including SE.} 
\vspace{-.5cm}
\label{fig:with_SE}
\end{figure}
In the figure above, the sample complexity and the average gaps are plotted for the Successive Elimination baseline (in red) as well as MODL, parents first, and the oracle methods. The SE baseline are almost too large to be comparable (roughly 200 times the other methods, all which appear comparatively as zero) and does not scale to more than a few variables. The sample complexity decreases with the number of parents as a greater portion of arms are able to be eliminated.

Figure~\ref{fig:var_parents} plots the same, without SE, for more interesting numbers of variables in four different combinations of $K=\{10,30\}$ and known/unknown $\calP_Y$ (the lower and upper bounds for $M_k$ are $3$ and $6$). As predicted by Theorem~\ref{thm:sample.complexity}, the sample complexity decreases with $\calP_Y$: many samples are required to distinguish between non-parents and a potential parents with $\theta_k^i\approx\epsilon$, so the complexity increases with the number of non-parents. Overall, the performance of \MODL is much closer to the performance of the oracle method. We see that the performance coincides when $\calP_Y = K$ since \MODL and the oracle method become the same algorithm. We also note that the P1 method does not benefit from $\calP_Y = K$. The P1 method also has consistently higher gaps (since, on occasion, it fails to identify a parent which would cause a large error), but all gaps are well within the desired error tolerance of $\epsilon = 1/2$. 
See the appendix for more figures: e.g.\ 
Figure~\ref{fig:var_hardness} plots the sample complexity for $K=30$ versus the support sizes $M_k$.

To summarize, we found that across all the settings that we investigated \MODL was substantially better than the P1 method, and its performance (in terms of the gap of the final action and the sample complexity) was closer to the performance of the oracle method than to the performance of the P1 method. Hence, we conclude that the penalty of not knowing the parents is relatively small and much smaller than the cost of learning the parents first.

\section{Discussion}
In this paper, we have proposed an approach to solving the causal bandit problem under the general setting of an unknown causal graph (\CBUG). Using the key insight that having no latent confounding between $Y$ and any of its ancestors implies that a global intervention is equivalent to an intervention on the optimal set $\pa(Y)$, we showed that an additive outcome assumption allows us to solve the \CBUG problem as a combinatorial linear bandit.

Limiting our algorithm to marginal action sets alleviated the computational burden by providing an easy approximation to the optimal design problem and a factorization of the action set. Two immediate direction for improving our algorithm and analysis is to consider the quality of approximation in the rejection procedure. A rejection procedure that outputs a marginal action set must reject fewer points than an unconstrained procedure. Are there principled ways of interpolating between marginal and full actions sets, perhaps using unions of marginal sets, which would let us trade-off computation and a larger number of rejected actions? The bound presented in this paper decomposes by variable, but a more nuanced rejection procedure and analysis should involve how the gaps between variables relate.

Other possible extensions include: (i) relaxing the additive outcome assumption, for example by adding ``interaction terms,'' (ii)  investigating what assumption allow for algorithms that adapt to sparsity, and (iii) considering the linear and continuous case. More generally, how can our method generalize to the case of latent confounders between $Y$ and its ancestors? 

Finally, we contrast our setting with the most common setting of causal inference where the goal is to learn the causal graph with as few experiments as possible. This setting is typical in the infinite data case, so the problem difficulty is measured in the number of experiments. As we are in the finite sample case, the number of statistical units (i.e.\ samples or rounds) is the most important quantity. With this distinction in mind, intervening on many variables is justified so long as we can reduce the sample complexity. However, budgeted versions of the causal bandit problem have been considered \citep{nair2021budgeted} and are another interesting direction.

\paragraph{Acknowledgements}
We wish to thank Eleni Sgouritsa for help in designing the code and technical support and Yasin Abbasi-Yadkori for technical feedback and many fruitful discussions.

\bibliography{bib}

\begin{thebibliography}{26}
\providecommand{\natexlab}[1]{#1}
\providecommand{\url}[1]{\texttt{#1}}
\expandafter\ifx\csname urlstyle\endcsname\relax
  \providecommand{\doi}[1]{doi: #1}\else
  \providecommand{\doi}{doi: \begingroup \urlstyle{rm}\Url}\fi

\bibitem[Allen-Zhu et~al.(2021)Allen-Zhu, Li, Singh, and Wang]{allen2021near}
Allen-Zhu, Z., Li, Y., Singh, A., and Wang, Y.
\newblock Near-optimal discrete optimization for experimental design: A regret
  minimization approach.
\newblock \emph{Mathematical Programming}, 186:\penalty0 439--478, 2021.

\bibitem[Bilodeau et~al.(2022)Bilodeau, Wang, and Roy]{bilodeau2022adaptively}
Bilodeau, B., Wang, L., and Roy, D.~M.
\newblock Adaptively exploiting d-separators with causal bandits.
\newblock In \emph{Advances in Neural Information Processing Systems}, 2022.

\bibitem[B{\"u}hlmann et~al.(2014)B{\"u}hlmann, Peters, and
  Ernest]{buhlmann2014cam}
B{\"u}hlmann, P., Peters, J., and Ernest, J.
\newblock Cam: Causal additive models, high-dimensional order search and
  penalized regression.
\newblock \emph{The Annals of Statistics}, 42\penalty0 (6):\penalty0
  2526--2556, 2014.

\bibitem[Chen et~al.(2014)Chen, Lin, King, Lyu, and
  Chen]{chen2014combinatorial}
Chen, S., Lin, T., King, I., Lyu, M.~R., and Chen, W.
\newblock Combinatorial pure exploration of multi-armed bandits.
\newblock In \emph{Advances in Neural Information Processing Systems},
  volume~27, 2014.

\bibitem[Constantinou \& Dawid(2017)Constantinou and
  Dawid]{constantinou2017extended}
Constantinou, P. and Dawid, A.~P.
\newblock Extended conditional independence and applications in causal
  inference.
\newblock \emph{Annals of Statistics}, 45\penalty0 (6):\penalty0 2618--2653,
  2017.

\bibitem[De~Kroon et~al.(2022)De~Kroon, Mooij, and Belgrave]{de2022causal}
De~Kroon, A., Mooij, J., and Belgrave, D.
\newblock Causal bandits without prior knowledge using separating sets.
\newblock In \emph{Conference on Causal Learning and Reasoning}, pp.\
  407--427. PMLR, 2022.

\bibitem[Du et~al.(2020)Du, Kakade, Wang, and Yang]{du2020good}
Du, S.~S., Kakade, S.~M., Wang, R., and Yang, L.~F.
\newblock Is a good representation sufficient for sample efficient
  reinforcement learning?
\newblock In \emph{International Conference on Learning Representations}, 2020.

\bibitem[Du et~al.(2021)Du, Kuroki, and Chen]{du2021combinatorial}
Du, Y., Kuroki, Y., and Chen, W.
\newblock Combinatorial pure exploration with full-bandit or partial linear
  feedback.
\newblock In \emph{Proceedings of the AAAI Conference on Artificial
  Intelligence}, volume~35, pp.\  7262--7270, 2021.

\bibitem[Even-Dar et~al.(2006)Even-Dar, Mannor, Mansour, and
  Mahadevan]{even2006action}
Even-Dar, E., Mannor, S., Mansour, Y., and Mahadevan, S.
\newblock Action elimination and stopping conditions for the multi-armed bandit
  and reinforcement learning problems.
\newblock \emph{Journal of Machine Learning Research}, 7\penalty0 (6):\penalty0
  1079--1105, 2006.

\bibitem[Fiez et~al.(2019)Fiez, Jain, Jamieson, and
  Ratliff]{fiez2019sequential}
Fiez, T., Jain, L., Jamieson, K.~G., and Ratliff, L.
\newblock Sequential experimental design for transductive linear bandits.
\newblock In \emph{Advances in Neural Information Processing Systems},
  volume~32, 2019.

\bibitem[Gabillon et~al.(2011)Gabillon, Ghavamzadeh, Lazaric, and
  Bubeck]{gabillon2011multi}
Gabillon, V., Ghavamzadeh, M., Lazaric, A., and Bubeck, S.
\newblock Multi-bandit best arm identification.
\newblock In \emph{Advances in Neural Information Processing Systems},
  volume~24, 2011.

\bibitem[Hastie(2017)]{hastie2017generalized}
Hastie, T.~J.
\newblock Generalized additive models.
\newblock In \emph{Statistical models in S}, pp.\  249--307. Routledge, 2017.

\bibitem[Imbens \& Rubin(2015)Imbens and Rubin]{imbens2015causal}
Imbens, G.~W. and Rubin, D.~B.
\newblock \emph{Causal inference in statistics, social, and biomedical
  sciences}.
\newblock Cambridge University Press, 2015.

\bibitem[Lattimore et~al.(2016)Lattimore, Lattimore, and
  Reid]{lattimore2016causal}
Lattimore, F., Lattimore, T., and Reid, M.~D.
\newblock Causal bandits: Learning good interventions via causal inference.
\newblock In \emph{Advances in Neural Information Processing Systems}, pp.\
  1181--1189, 2016.

\bibitem[Lattimore \& Szepesv{\'a}ri(2020)Lattimore and
  Szepesv{\'a}ri]{lattimore2020bandit}
Lattimore, T. and Szepesv{\'a}ri, C.
\newblock \emph{Bandit Algorithms}.
\newblock Cambridge University Press, 2020.

\bibitem[Lee \& Bareinboim(2018)Lee and Bareinboim]{lee2018structural}
Lee, S. and Bareinboim, E.
\newblock Structural causal bandits: where to intervene?
\newblock In \emph{Advances in Neural Information Processing Systems}, pp.\
  2568--2578, 2018.

\bibitem[Lu et~al.(2020)Lu, Meisami, Tewari, and Yan]{lu2020regret}
Lu, Y., Meisami, A., Tewari, A., and Yan, W.
\newblock Regret analysis of bandit problems with causal background knowledge.
\newblock In \emph{Conference on Uncertainty in Artificial Intelligence}, pp.\
  141--150. PMLR, 2020.

\bibitem[Lu et~al.(2021)Lu, Meisami, and Tewari]{lu2021causal}
Lu, Y., Meisami, A., and Tewari, A.
\newblock Causal bandits with unknown graph structure.
\newblock In \emph{Advances in Neural Information Processing Systems},
  volume~34, pp.\  24817--24828, 2021.

\bibitem[Maeda \& Shimizu(2021)Maeda and Shimizu]{maeda2021causal}
Maeda, T.~N. and Shimizu, S.
\newblock Causal additive models with unobserved variables.
\newblock In \emph{Uncertainty in Artificial Intelligence}, pp.\  97--106,
  2021.

\bibitem[Maiti et~al.(2022)Maiti, Nair, and Sinha]{maiti2022causal}
Maiti, A., Nair, V., and Sinha, G.
\newblock A causal bandit approach to learning good atomic interventions in
  presence of unobserved confounders.
\newblock In \emph{Uncertainty in Artificial Intelligence}, pp.\  1328--1338.
  PMLR, 2022.

\bibitem[Nair et~al.(2021)Nair, Patil, and Sinha]{nair2021budgeted}
Nair, V., Patil, V., and Sinha, G.
\newblock Budgeted and non-budgeted causal bandits.
\newblock In \emph{International Conference on Artificial Intelligence and
  Statistics}, pp.\  2017--2025. PMLR, 2021.

\bibitem[Pearl(2000)]{pearl2000causality}
Pearl, J.
\newblock \emph{Causality: Models, Reasoning, and Inference}.
\newblock Cambridge University Press, 2000.

\bibitem[Soare et~al.(2014)Soare, Lazaric, and Munos]{soare2014best}
Soare, M., Lazaric, A., and Munos, R.
\newblock Best-arm identification in linear bandits.
\newblock In \emph{Advances in Neural Information Processing Systems},
  volume~27, 2014.

\bibitem[Tao et~al.(2018)Tao, Blanco, and Zhou]{tao2018best}
Tao, C., Blanco, S., and Zhou, Y.
\newblock Best arm identification in linear bandits with linear dimension
  dependency.
\newblock In \emph{International Conference on Machine Learning}, pp.\
  4877--4886, 2018.

\bibitem[Xiong \& Chen(2023)Xiong and Chen]{xiong2023combinatorial}
Xiong, N. and Chen, W.
\newblock Combinatorial pure exploration of causal bandits.
\newblock In \emph{The Eleventh International Conference on Learning
  Representations}, 2023.

\bibitem[Xu et~al.(2018)Xu, Honda, and Sugiyama]{xu2018fully}
Xu, L., Honda, J., and Sugiyama, M.
\newblock A fully adaptive algorithm for pure exploration in linear bandits.
\newblock In \emph{International Conference on Artificial Intelligence and
  Statistics}, pp.\  843--851, 2018.

\end{thebibliography}
\bibliographystyle{icml2023}

\newpage
\appendix
\onecolumn

\section{A simple test for the parents of \texorpdfstring{$Y$}{Y}}\label{appendix:parents.test}
In the experiments, we compared MODL with an algorithm that first learns the parents of $Y$ then ran a bandit algorithm on the remaining variables.

The algorithm is intuitively simple: we will construct confidence intervals of width $\epsilon/2$ for $\ex[Y\cond\doi(X_k=j, \X_{-k} = \x_{-k})]$ for all values of $j$, and if the intersection of the confidence intervals do not overlap, then one of the values of $X_k$ is statistically significantly different from the others. Because we have intervened to fix all the other variables, this difference must be because $X_k$ is a parent of $Y$. Formally, we have the following lemma.
\begin{lemma}
Let $\x_0\in \supp(\X)$ be fixed, and $\x_{-k}$ be $\x_0$ with the $k$th variable's value removed. For $\delta \in (0,1)$ and $\epsilon>0$, assume that $Y_k^1,\ldots, Y_k^{n_k}\sim p(Y\cond\doi(X_k=j, \X_{-k} = \x_{-k})$, where
\[
    n_k \df \left\lceil
    \frac{8\sigma^2}{\epsilon^2}\logf{2 K \supp(X_k)}{\delta}
    \right\rceil.
\]
Then 
\[
\prob\left(
\forall 1\leq k\leq K, j\in\supp(X_k),
\left| \frac{\sum_{i=1}^{n_k} Y_k^i}{n_k} - 
\ex[Y\cond\doi(X_k=j, \X_{-k} = \x_{-k})]\right|
\leq \frac{\epsilon}{2} 
\right)\geq 1-\frac{1}{\delta}.
\]
Hence, by the union bound, all the confidence intervals are simultaneously correct with probability at least $1-\delta$.

\end{lemma}
\begin{proof}
With the $n_k$ defined in the lemma, we can verify that
\[
\sqrt{\frac{2\sigma^2}{n_k}\logf{2 K \supp(X_k)}{\delta}} \leq \frac{\epsilon}{2}.
\]
Thus, Lemma~\ref{lem:azuma} implies that
\[
\prob\left(
\left| \frac{\sum_{i=1}^{n_k} Y_k^i}{n_k} - 
\ex[Y\cond\doi(X_k=j, \X_{-k} = \x_{-k})]\right|
\leq \frac{\epsilon}{2}
\right)\geq 1-\frac{2 K \supp(X_k)}{\delta}.
\]
Finally, by the union bound, all the confidence intervals are simultaneously correct with probability at least $1-\delta$, as claimed.
\end{proof}

Under the event that the bounds are all correct, $X_k$ is a parent of $Y$ if there exist two intervals that do not overlap; in this case, the means of the two interventions must be different with high probability. See Algorithm~\ref{alg:parents} for pseudocode. In addition to having few false positives, we can argue that the algorithm has few false negatives, as presented in the following lemma.

\begin{algorithm}
\caption{Finding $\pa(Y)$}
\label{alg:parents}
\begin{algorithmic}
\STATE \textbf{Given:} $\epsilon>0$, $\delta\in(0,1)$, $\sigma^2$, $\x_0\in\supp(\X_{[K]})$ .
\STATE 
$\widehat\pa(Y)\leftarrow \emptyset$
\FOR{$k=1,\ldots, K$:}
    \STATE $C_k\leftarrow\Reals$
    \STATE $n\leftarrow \left\lceil
    \frac{8\sigma^2}{\epsilon^2}\logf{2 K \supp(X_k)}{\delta}
    \right\rceil.$
    \FOR{$j = 1,\ldots, M_k$}
        \STATE  Collect $y^1,\ldots, y^n \sim p(Y\cond\doi(X_k=j, \X_{-i} = \x'))$
        \STATE $C_k\leftarrow C_k \cap \left(\frac{\sum_{i=1}^n y^i}{n} - \frac{\epsilon}{2},
        \frac{\sum_{i=1}^n y^i}{n} + \frac{\epsilon}{2}\right)$        
        \IF{$C_k = \emptyset$}
        \STATE $\widehat\pa(Y)\leftarrow\widehat\pa(Y) \cup \{ X_k\}$
        \STATE Skip the rest of the tests for $X_k$.
        \ENDIF
        \ENDFOR
    \ENDFOR
\STATE Return $\widehat\pa(Y)$
\end{algorithmic}
\end{algorithm}

\begin{lemma}
Assume that for all $X_k\in\pa(Y)$, there exists $i,j\in\supp(X_k)$ such that $|f_k(i) - f_k(j)|\geq \epsilon$. Then, with probability $1-\delta$, Algorithm~\ref{alg:parents} correctly recovers the parents.
\end{lemma}

\section{Additional Experiments}\label{appendix:experiments}
This section holds additional experiments under the same setup in Section~\ref{sec:experiments}.
\begin{figure*}[t]
\centering
\includegraphics[width=14.3cm,height=3.5cm]{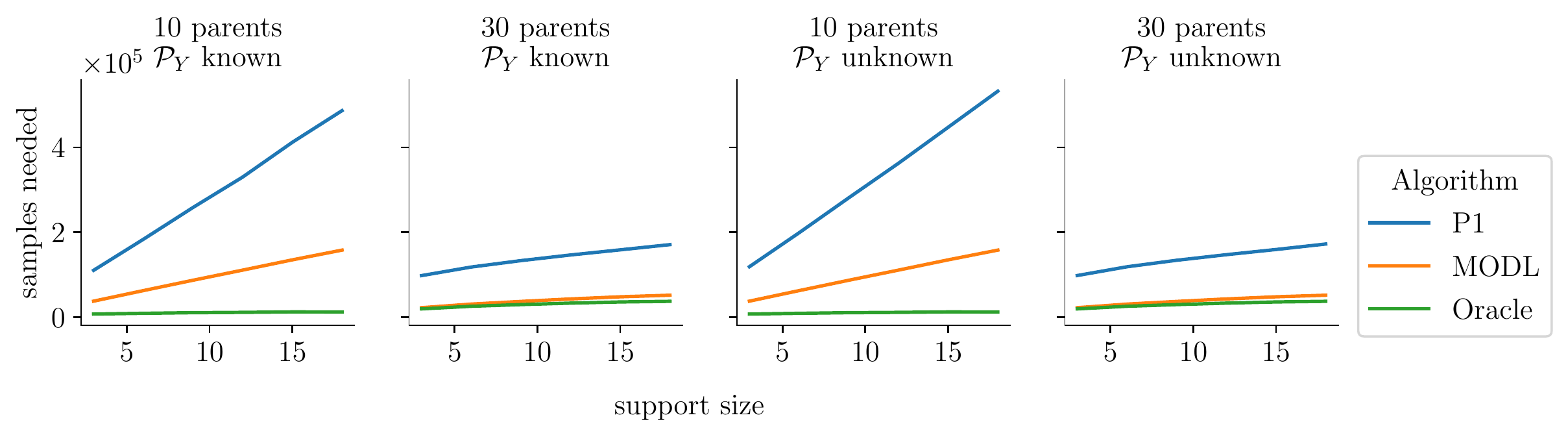}
\caption{Sample complexity versus lower bound on $M_k$.} 
\label{fig:var_hardness}
\end{figure*}

Figure~\ref{fig:var_hardness} plots the sample complexity for $K=30$ versus the support sizes $M_k$, where the x-axis specifies the lower bound on $M_k$ (the upper bound is 3 larger). As above, the performance of \MODL is much closer to the performance of the oracle method than to the performance of the P1 method, and almost identical for $\calP_Y = K = 30$.

Figure~\ref{fig:varying_degree} plots the sample complexity as the degree of the sampled graph changes, when $\calP_Y = 10$ and $K=30$. Unsurprisingly, the degree has very little effect on the sample complexity, confirming our intuition from causal inference that intervening on all parents renders the rest of the causal graph unimportant.
\begin{figure}
\centering
  \includegraphics[height=7cm]{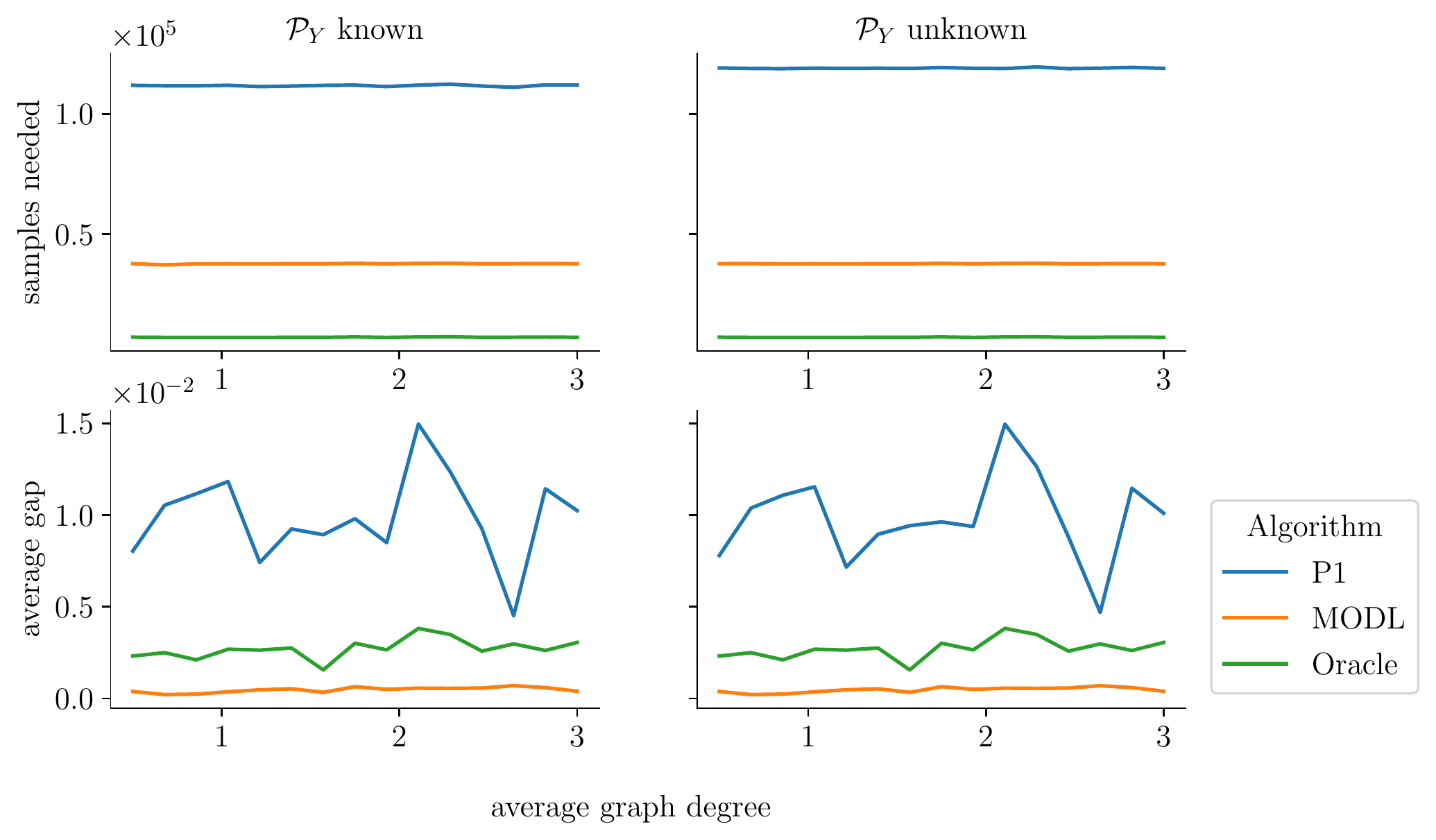}
\caption{Sample complexity and average gap vs. number of variables with 30 variables and 10 parents.} 
\label{fig:varying_degree}
\end{figure}

Figure~\ref{fig:model_mis} confirms our suspicion that the linear bandit is very sensitive to model mispecification. We generated non-linear data by using the outcome model
\[
Y = \sum_{k=1}^{\calP_Y} f_k(X_k) + \alpha B M_{\max}^{-4}
\left( X_{i_1} X_{i_2} X_{i_3} X_{i_4}
+  X_{j_1} X_{j_2} X_{j_3}
+ X_{k_1} X_{k_2}
\right)
\]
for randomly chosen (without replacement) indices $i_1, i_2, i_3, i_4$, $j_1, j_2, j_3$, and $k_1, k_2$ from $[\calP_Y]$ and $M_{\max}$ is the upper bound on the support size (6 in this case). This model was chosen to resemble the effect of adding ``interaction terms'' that are the product of several variables. The leading coefficient is chosen to keep the maximum interaction term roughly $\alpha B$ so choosing $\alpha\in[0,1]$ keeps the scale of the interactions terms roughly equivalent to the additive terms. Despite this scaling, we still find that the performance is very sensitive to model mismatch, as illustrated in Figure~\ref{fig:model_mis}. We also note that the parents-first approach is much more sensitive to model mispecification, at least in the average gap, because the mispecification increases the probability of the parents being misspecified, which in turn causes a large error. MODL and the oracle algorithms are more immune to this effect.
\begin{figure}[t]
\centering
  \includegraphics[height=7cm]{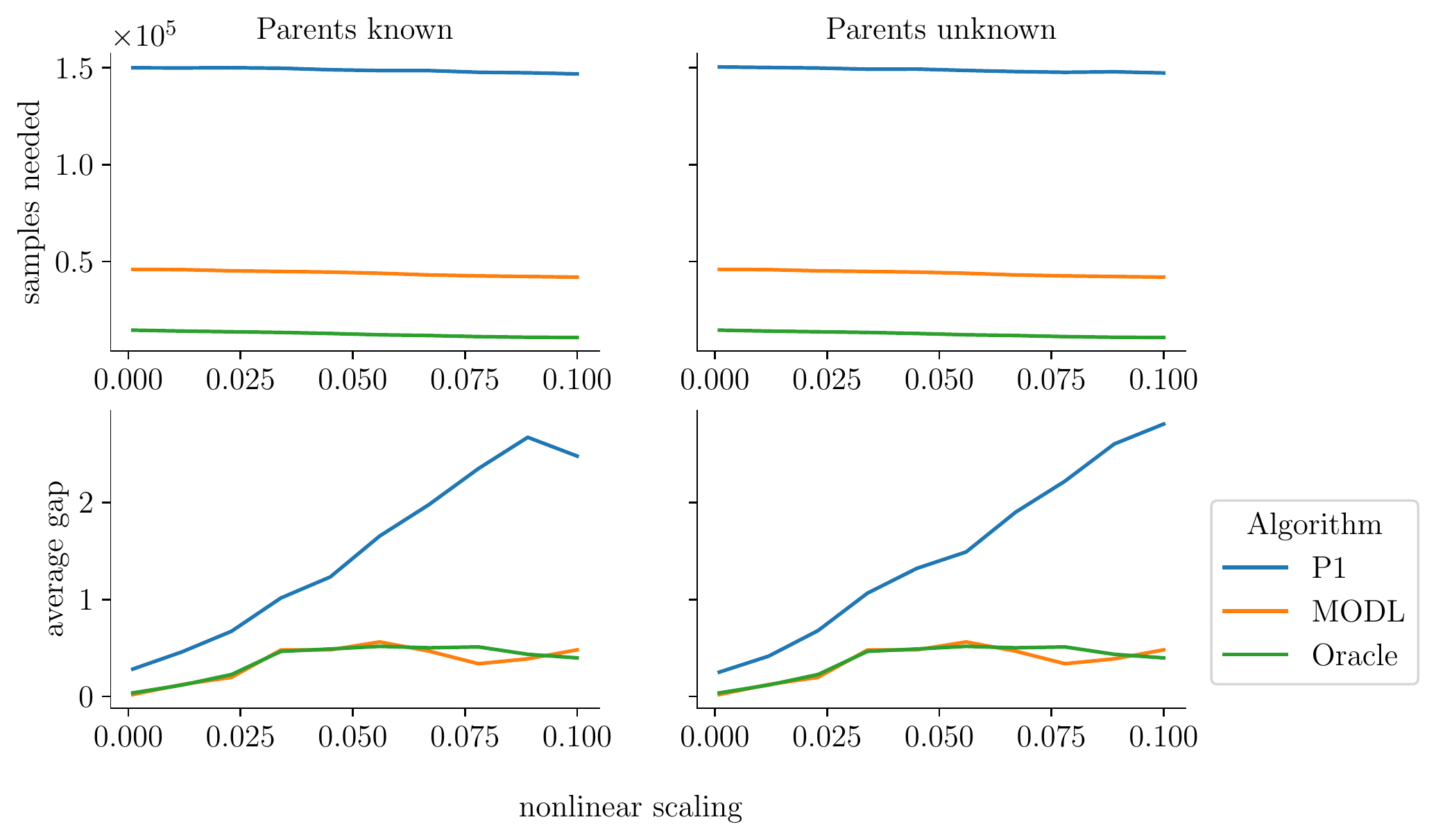}
\caption{Sample complexity and average gap vs. model mispecification. The x-axis details the coefficient of a multiplicative nonlinear term in the expected response function.} 
\label{fig:model_mis}
\end{figure}

\section{Proofs}\label{appendix:proofs}
\begin{proof}[Proof of Theorem~\ref{thm:lower.bound}]
This lower bound can by shown by invoking Theorem~33.5 from \cite{lattimore2020bandit}.

For any $\x\in\supp(\X_{[K]})$, we define $q(\x)$ to be the $\x$ values of $\pa(Y)$; in particular, the reward of any two actions $\x$ and $\x'$ are equal if $q(\x)=q(\x')$. We invoke the theorem for a bandit with $\supp(\X_{[K]})$ arms, one for each action, and with a set $\calE$ of bandit environments indexed by $\x_\pa\in\supp(\pa(Y))$; for a bandit $\nu(\x_pa)\in\calE$ corresponding to $\x_\pa$, we set $Y\cond\doi(\X=\x)\sim\text{Bernoulli}(\frac{1}{2} + \epsilon \indicator{q(\x)=\x_\pa})$. We can explicitly calculate $\calE_{alt}(\nu(\x_\pa))=\{\nu(\x_\pa'): \x_\pa'\neq \x_\pa\}$, which are the set of bandit environments with a different optimal arm than $\nu$'s. We also use $\nu_x$ to indicate the reward distribution of action $\x$ under bandit $\nu$. Then, examining the terms in the theorem, we need to calculate
\begin{align*}
c(\nu)
&=
\max_{\alpha\in\triangle_{\supp(\X_{[K]})}}
\left(
    \min_{\nu'\in\calE_{alt}(\nu)} \sum_{\x\in\supp(\X_{[K]})}\alpha_{\x} D(\nu_\x,\nu_\x')
\right),
\end{align*}
where $D(\cdot,\cdot)$ is the relative entropy.
If $\nu'\in\calE_{alt}(\nu)$, this means that the $\x_\pa'$ corresponding to $\nu'$ is different from the $\x_\pa$ corresponding to $\nu$. Fix one $\x_\pa'$; we can see that $D(\nu_\x,\nu_\x')$ is zero for all $\x, \x'$ with $q(\x) = q(\x')$. We also use the fact that  $D(\nu_\x,\nu_\x') = O(\epsilon^2)$ when $\epsilon$ is small. Hence 
\[
    \min_{\nu'\in\calE_{alt}(\nu)} \sum_{\x\in\supp(\X_{[K]})}\alpha_{\x}     D(\nu_\x,\nu_\x')
    =
    \min_{\x_\pa'\neq \x_\pa\in \supp(\pa(Y))}
    O(\epsilon^2)\alpha_{\x_\pa'},
\]
where $\x_\pa$ corresponds to $\nu$. Taking the max over $\alpha$, we see that any $\alpha$ must spread mass evenly across all $\x$ with $\x_\pa'\neq \x_\pa\in \supp(\pa(Y))$, which leads to $c(v) = O\left(\frac{\epsilon^2}{|\supp(\pa(Y))|}\right)$. Combining these calculations with the theorem, we find that
\[
\ex[\tau] \geq O\left(\frac{|\supp(\pa(Y))|}{\epsilon^2}\logf{1}{\delta}\right),
\]
where $\ex[\tau]$ is the expected stopping time of any sound (i.e.\ $(\epsilon,\delta)$-PAC) algorithm, as claimed.

We could also follow the techniques from the lower bound of \cite{du2020good} and reduce the index-query problem to the \CBUG problem. 
\end{proof}

\begin{proof}[Proof of Lemma~\ref{lem:optimal.design.solution}]
Let $x^1,\ldots, x^n$ be a sequence of actions. To calculate $V_n$, we write 
\[
  V_n = \sum_{i=1}^n 
  (e_1(\x_1^i),\ldots, e_K(\x_K^i))
  (e_1(\x_1^i),\ldots, e_K(\x_K^i))^\top
  = \sum_{i=1}^{n} \sum_{j=1}^K \sum_{j'=1}^K
  e_j(\x_j^i)e_{j'}(\x_{j'}^i)
  =
  D_n
  + 
  C_n,
\]
where $D_n$ is the matrix of on-diagonal components and $C_n$ contains all the off-diagonal terms; both are positive semi-definite. Using the fact that if $A$ and $B$ are PSD matrices, then $x^\top(A + B)^{\dagger}x\leq x^\top A^{\dagger} x$, we can upper bound $\mnorm{e(\x)-e(\x')}{V_n} $ by 
 $\mnorm{e(\x)-e(\x')}{D_n}$. 

Letting $N_k^j = \sum_{t\leq n}\indicator{\x_k^t = j}$ be the round $\x_k^t$ took on the $j$th value, we can show that 
\[
    D_n
    =
    \sum_{t\leq n} e(\x^t)e(\x^t)^\top 
    = 
    \sum_{t\leq n} \diag( (e_1(\x_1^t),\ldots, e_K(\x_K^t))
    =
    \diag\left(
    N_1^1, N_1^2,\ldots, N_1^{M_1},
    N_2^1,\ldots, N_K^{M_K}
    \right);
\]
$D_n$ is the diagonal matrix of the counts of the values. We can upper bound the optimal design problem over $\x^{[n]}$ with another problem over counts of values that appears in $\calS$ and add to $n$. Formally, this set is 
\[
    \calN(\calS,n)\df \left\{(N_1^1,N_1^2,\ldots,N_K^{M_K}): \forall k, N_k^j=0\text{ if } j\notin \calS_k \text{ and } \sum_j N_k^j=n\right\}.
\]

We can easily check that, for any $\x, \x'\in\calS$, 
\begin{align*}
    \mnorm{e(\x)-e(\x')}{D_n} 
    &=
    (e(\x)-e(\x'))^\top V_n^{\dagger} (e(\x)-e(\x'))\\
    &=
    \sum_k
    (e_k(\x_k) - e_k(\x_k'))^\top
    \left(\frac{e_k(\x_k)}{N_k^{\x_k}} - \frac{e_k(\x_k')}{N_k^{\x_k'}}\right)\\
    &=
    \sum_k
    \left(\frac{1}{N_k^{\x_k}}
    + \frac{1}{N_k^{\x_k'}}\right)\indicator{\x_k \neq \x_k'}.
\end{align*}
Thus, the $\mathcal{XY}$-optimal has an upper bound
\begin{align*}
    \x^{\mathcal{XY}}(\mathcal S, n)
    &=
    \arg\min_{\{N_i^j\}\in \calN(\calS,n)}
    \max_{\x, \x'\in\calS} \mnorm{e(\x)-e(\x')}{V_n}\\
    &=  
    \arg\min_{\{N_k^j\}\in \calN(\calS,n)}
    \sum_k \max_{j,j'\in\calS_k}\frac{1}{N_k^{j}} + \frac{1}{N_k^{j'}}
    \\
    &=  
    \arg\min_{\{N_k^j\}\in \calN(\calS,n)}
    \sum_k
    \frac{2}{\min_{j\in\calS_k}N_k^{j}}.
\end{align*}
We can solve the problem in closed from: for all $i\leq V$, allocate $N_i^j$ evenly among all $j\in \calS_i$. Because $|\calS_i|$ may not divide $n$, we may have rounding errors and can only guarantee that $(N_i^j)^{-1}\in \left[ \lfloor n/S_i \rfloor, \lceil n/S_i \rceil \right]$, which results in an objective value of
$\sum_i 
\frac{2}{\lfloor n/S_i \rfloor}\leq 2\sum_i \frac{|S_i|}{n - |S_i|}$.

\end{proof}

\begin{namedthm*}{Theorem~\ref{thm:sample.complexity}}
Algorithm~\ref{algorithm:MODL} is $(\epsilon,\delta)$-PAC. The expected sample complexity is
\begin{align*}
    H^\epsilon 
    &\df
    \frac{16}{3}\sigma^2 \logf{\log(B K/\epsilon)}{\delta}
    \left(
        \sum_{k\in\pa(Y)}
  \sum_{i=1}^{M_k}  
  \frac{1}{(\Delta_k^i\wedge(\epsilon/K))^2}  
  + 
  \sum_{k\notin\pa(Y)} M_k \frac{1}{\epsilon^2}
  \right). 
\end{align*}
If the number of parents $\calP_Y$ is provided, the complexity is instead
\[
    H^{\epsilon, \calP_Y} =     \frac{16}{3}\sigma^2 \logf{\log(BK/\epsilon)}{\delta}
    \sum_{k=1}^K
    \sum_{i=1}^{M_k}
    \frac{K^2}{(\Delta_{\min} \wedge \Delta_k^i\wedge(\epsilon/K))^2},
\]
where $\Delta_{\min} = \min_{k\leq \calP_Y}\min_{i\in [M_k]}\Delta_k^i$ is the minimum gap in the parents.
\end{namedthm*}
\begin{proof}[Proof of Theorem~\ref{thm:sample.complexity}]
Recall that MODL alternates between two stages, data-collection and action elimination, and uses an exponentially decreasing error tolerance. Let $\calS_k(\ell)$, $\gamma(\ell)$, and $n_
ell$ be the corresponding the values during phase $\ell$.

It is easy to verify that $\gamma \approx B/2$ for $\ell=1$  and $\gamma = \frac{\epsilon}{2K}$ for $\ell=L$. We will show that, with the chosen $n_\ell$,
\[
    P\left(
    \langle\hat\theta - \theta^*, e(\x) - e(\x')\rangle
    \leq
    \gamma(\ell)
    \forall \x, \x'\in\calS(\ell), \ell \in [L]\right)\geq 1-\delta.
\]
That is, with high probability, all our confidence intervals used by the algorithm are correct.

Let $\calS_k(\ell)$ be the $k$th marginal of $\calS$ during phase $\ell$ of the algorithm. By choosing 
\[
    n_\ell = \left\lceil \frac{4\sigma^2 |\sum_k\calS_k(\ell)|}{\gamma^2}\logf{L}{\delta}\right\rceil,
\]
Lemma~\ref{lem:optimal.design.solution} guarantees that $\max_{z,z'\in\calS}\mnorm{e(\x) - e(\x')}{V_n}\leq \sum_k \frac{2 |\calS_k(\ell)|}{n}$. Lemma~\ref{lem:azuma} provides
\[
    \langle\hat\theta - \theta^*, e(\x) - e(\x')\rangle
    \leq
    \sqrt{2\sigma^2\mnorm{e(\x) - e(\x')}{V_n}\logf{L}{\delta}}
    \leq
    \sqrt{\frac{4\sigma^2\sum_k |\calS_k(\ell)|}{n}\logf{1}{\delta}}
    \leq
    \gamma.
\]
with probability at least $1-\delta$ for all $\x\in\calS(\ell)$ simultaneously.

At each stage $\ell$, the elimination algorithm proceeds only eliminating $\x$ where there exists $\x'$ with $\langle \hat\theta, e(\x') - e(\x)\rangle \geq \gamma(\ell)$. If such a $\x'$ exists, then we can conclude that
\begin{align*}
    \langle \theta^*, e(\x)-e(\x^*)\rangle
    \leq
    \langle \theta^*, e(\x) - e(\x')\rangle
    \leq 
    \langle \hat\theta, e(\x) - e(\x')\rangle - \gamma(\ell)
    \leq 
    0.
\end{align*}
Hence, we have shown that, for all stages $\ell$, the algorithm never eliminates any action that is $\gamma(\ell)$-suboptimal.

The last step to checking correctness is to show that an $\epsilon$-suboptimal action is returned. In round $\ell=L$, Lemma~\ref{lem:azuma},  guarantees that 
$\langle \theta^*, e(\x) - e(\x')\rangle\leq\epsilon/2$ for all $\x,\x'$ in $\calS(\ell)$. Applying this to the action $\hat\x$ returned by the algorithm and the fact that, under the event that the confidence intervals are correct, $\x^*\in \calS(\ell)$, we have
\begin{align*}
    \langle \theta^*, e(\hat\x)-e(\x^*)\rangle
    \leq
    \langle \theta^*-\hat\theta, e(\hat\x)-e(\x^*)\rangle  
    + \langle \hat\theta, e(\hat\x)-e(\x^*)\rangle  
    \leq 
    \frac{K\epsilon}{2K}+\frac{\epsilon}{2},
\end{align*}
where the last step used the fact that each variable's error was controlled to $\epsilon/2k$.

The total sample complexity is
\[
\sum_{\ell=1}^L \sum_{k=1}^K \left|\calS_k(\ell)\right|
\frac{4\sigma^2 }{\gamma(\ell)^2}\logf{L}{\delta}.
\]

We now look to bound $\sum_{\ell=1}^L \sum_{k=1}^K \left|\calS_k(\ell)\right|$ in terms of instance-dependent quantities. With no bound on $|\pa(Y)|$,  the complexity also decomposes. Let $G = \{\gamma_1>\ldots>\gamma_L = \epsilon/2K\}$ be the set of $\gamma$ used by the algorithm. Setting $c = 4\sigma^2 \log(L/\delta)$ and using the fact that $\hat\Delta_k^i$ are all unbiased, the expected sample complexity can be upper bounded by calculating
\begin{align*}
    \sum_{k=1}^K 
    \sum_{\ell=1}^L 
    \frac{c}{\gamma(\ell)^2}
    \left|\left\{\Delta_k^i\leq \gamma(\ell)
    \right\}\right|
&=
    \sum_{k=1}^K 
    \sum_{\gamma\in G}
    \sum_{i=1}^{M_k}
    \frac{c}{\gamma^2}
    \indicator{\Delta_k^i\leq \gamma}\\
&=
    \sum_{k=1}^K 
    \sum_{i=1}^{M_k}
    \sum_{\{\gamma\in G: \gamma\geq \Delta_k^i\wedge \epsilon/K\}}
    \frac{c}{\gamma^2}\\
&\leq
    \sum_{k=1}^K
    \sum_{i=1}^{M_k}
    \sum_{j\geq 0}
    \frac{c}{2^{2j}\left(\Delta_k^i\wedge\frac{\epsilon}{K}\right)^2}
=
    \sum_{k=1}^K 
    \sum_{i=1}^{M_k}
    \frac{4c}{3\left(\Delta_k^i\wedge\frac{\epsilon}{K}\right)^2}
    \\
&=
  \sum_{k\in\pa(Y)}
  \sum_{i=1}^{M_k}
  \frac{4c}{3(\Delta_k^i\wedge\frac{\epsilon}{K})^2}  
  + 
  \sum_{k\notin\pa(Y)} M_k \frac{4cK^2}{3\epsilon^2}  
\end{align*}

To summarize, the sample complexity is the sum the sample complexity of the parents with a linear term for the non-parents. Intuitively, it is very difficult to differentiate between a parent with some $|f_k(\cdot)|\geq \epsilon$ and a non-parent, so we expect to see these terms in the sample complexity. 

Recall that $\Delta_{\min} = \min_{k\in\pa(Y)}\min_{i\in [M_k]}\Delta_k^i$. When the number of parents is given, the algorithm will terminate as soon as $\gamma(\ell) \leq \Delta_{\min}$. Hence, the sample complexity can be decomposed into the samples needed to learn the parents, 
$\sum_{k\in\pa(Y)}
\sum_{i=1}^{M_k}
\frac{4c}{3(\Delta_k^i\wedge\epsilon)^2}$, and the extra samples from all the remaining variables 
\begin{align*}
    \sum_{\{\gamma\in G: \gamma \geq \Delta_{\min} \wedge \epsilon\}}
    \sum_{k\notin\pa(Y)} M_k \frac{c}{\gamma^2}
    &=
    \sum_{k\notin\pa(Y)} M_k c
    \sum_{\{\gamma\in G: \gamma \geq  \Delta_{\min} \wedge \epsilon\}}\frac{1}{\gamma^2}\\
    &\leq
    \sum_{k\notin\pa(Y)} M_k \frac{4c}{3(\Delta_{\min} \wedge \epsilon)^2}.
\end{align*}
Hence, the total sample complexity is 
\[
\frac{4c}{3}
\sum_{k\in\pa(Y)}
\sum_{i=1}^{M_k}
\frac{1}{(\Delta_k^i\wedge\epsilon)^2}
+
\sum_{k\notin\pa(Y)} \frac{M_k}{\left(\Delta_{\min} \wedge \frac{\epsilon}{K}\right)^2}
=
\frac{4c}{3}
\sum_{k=1}^K
\sum_{i=1}^{M_k}
\frac{1}{\left(\Delta_{\min} \wedge \Delta_k^i\wedge\frac{\epsilon}{K}\right)^2}.
\]
\end{proof}

\begin{proof}[Proof of Theorem~\ref{thm:recovering.parents}]
We need to prove two things: first, all parents are discovered, and second, no erroneous parents are included. Throughout, we condition of the event that the confidence intervals are all correct, which happens with probability at least $1-\delta$, and demonstrated in the proof of Theorem~\ref{thm:sample.complexity}. 

Recalling that $\hat\theta_k^i(\ell)$ and $\gamma(\ell)$ are the respective quantities at phase $\ell$ of the algorithm, we argue that non-parents are not added to $\hat\pa(Y)$. Consider some non-parent $k>\pa_Y$. The concentration inequality from Lemma~\ref{lem:azuma}, applied to $\x, \x'$ that only differ in that the first corresponds to $X_k=i$ and the second to $X_k=j$, yields 
\[
    \langle \hat\theta - \theta^*, e(\x)- e(\x')\rangle\leq\gamma
    \Rightarrow 
    \hat\theta_k^i - \theta_k^j\leq\gamma.
\]
Taken with the fact that $\theta_k^i - \theta_k^j = 0$  for all $i,j\in M_k$, we have $\hat\theta_k^i(\ell)- \hat\theta_k^j(\ell)\leq 2\gamma(\ell)$, so $k$ is not added to $\hat\pa(Y)$ on the event that the confidence intervals are collect.

To show that all parents are included, we consider two cases. First, assume that the algorithm does not terminate early so $\gamma_L = \epsilon / 2$. For every $k\leq \calP_Y$, the algorithm must have found the $i, i'$ satisfying $|f_k(i) - f_k(i')|\geq \epsilon_{\min} \geq \epsilon/2$, so $k$ in included in $\hat\pa_Y$. 

The second case is when the algorithm terminates early. In this case, there must be  $\overline \calP_Y$ variables with only one action remaining. These variables must be parents, since, with high probability no non-parents are added. 

Under the event that $\hat\pa(Y) = \pa(Y)$, Fact~\ref{fact2} implies that the expected response of actions $\hat\x$ and $\hat \x_{\hat\pa(Y)}$ are equal, completing the proof.
\end{proof}

\end{document}